\documentclass[letterpaper]{article} 
\usepackage{aaai23}  
\usepackage{times}  
\usepackage{helvet}  
\usepackage{courier}  
\usepackage[hyphens]{url}  
\usepackage{graphicx} 
\usepackage{subfigure}
\usepackage{thm-restate}

\usepackage{natbib}  
\usepackage{caption} 
\usepackage{algorithm}
\usepackage{algorithmic}
\newlength{\algofontsize}
\usepackage{booktabs, multicol, multirow}       
\usepackage{amsfonts}       
\usepackage{amsmath}
\usepackage{amsthm}
\usepackage{xfrac}
\usepackage{textcomp}
\usepackage{newfloat}
\usepackage{listings}
\DeclareCaptionStyle{ruled}{labelfont=normalfont,labelsep=colon,strut=off} 
\floatstyle{ruled}
\newfloat{listing}{tb}{lst}{}
\floatname{listing}{Listing}
\title{Optimal Transport with Tempered Exponential Measures}
\author {
    Ehsan Amid\textsuperscript{\rm 1}\thanks{Authors listed alphabetically.},
    Frank Nielsen\textsuperscript{\rm 2},
    Richard Nock\textsuperscript{\rm 3},
    Manfred K. Warmuth\textsuperscript{\rm 3}
}
\affiliations {
    \textsuperscript{\rm 1} Google DeepMind\\
    \textsuperscript{\rm 2} Sony Computer Science Laboratories Inc.\\
    \textsuperscript{\rm 3} Google Research\\
    eamid@google.com, frank.nielsen@acm.org, \{richardnock, manfred\}@google.com
}

\usepackage{bibentry}

\usepackage{xcolor}
\usepackage{array}

\usepackage{mathabx}

\usepackage{booktabs,makecell} 

\usepackage{verbatim}

\usepackage{xfp}
\usepackage{dsfont}

\newcolumntype{?}{!{\vrule width 2pt}}

\newcolumntype{C}[1]{>{\centering\let\newline\\\arraybackslash\hspace{-5pt}}m{#1}}

\usepackage{upgreek,eucal,colortbl}
\usepackage{graphicx} 
\usepackage{multirow}
\usepackage{bm} 
\usepackage{color}
\usepackage{nicefrac}       

\usepackage{scalerel,stackengine}
\newcommand\reallywidecheck[1]{%
\savestack{\tmpbox}{\stretchto{%
  \scaleto{%
    \scalerel*[\widthof{\ensuremath{#1}}]{\kern-.6pt\bigwedge\kern-.6pt}%
    {\rule[-\textheight/2]{1ex}{\textheight}}
  }{\textheight}%
}{0.5ex}}%
\stackon[1pt]{#1}{\scalebox{-1}{\tmpbox}}%
}

\usepackage{tikz}
\usetikzlibrary{arrows,decorations.pathmorphing,backgrounds,positioning,fit,petri}
\usetikzlibrary{calc}

\usepackage{stmaryrd}

\definecolor{ao}{rgb}{0.0, 0.5, 0.0}

\makeatletter\newcommand*\bigcdot{\mathpalette\bigcdot@{.5}}
\newcommand*\bigcdot@[2]{\mathbin{\vcenter{\hbox{\scalebox{#2}{$\m@th#1\bullet$}}}}}
\makeatother

\definecolor{darkgreen}{RGB}{0,205,0}

\usepackage{pifont}

\usepackage{xcolor,colortbl}

\definecolor{Gray}{gray}{0.85}
\definecolor{LightCyan}{rgb}{0.88,1,1}

\definecolor{Gray}{gray}{0.85}
\definecolor{LightCyan}{rgb}{0.88,1,1}

\newcolumntype{a}{>{\columncolor{Gray}}c}
\newcolumntype{b}{>{\columncolor{white}}c}
\newcolumntype{d}{>{\columncolor{Gray}}r}

  \usepackage[framemethod=TikZ]{mdframed}
\mdfdefinestyle{MyFrame}{%
    linecolor=gray,
    outerlinewidth=2pt,
    roundcorner=10pt, 
    innertopmargin=0pt,
    innerbottommargin=\baselineskip,
    innerrightmargin=10pt,
    innerleftmargin=10pt,
    backgroundcolor=gray!10!white}

\mdfdefinestyle{MyFrameLoss}{%
    linecolor=gray,
    outerlinewidth=2pt,
    roundcorner=10pt, 
    innertopmargin=0pt,
    innerbottommargin=\baselineskip,
    innerrightmargin=10pt,
    innerleftmargin=10pt,
    backgroundcolor=white!60!cyan}

\mdfdefinestyle{MyFrameModel}{%
    linecolor=gray,
    outerlinewidth=2pt,
    roundcorner=10pt, 
    innertopmargin=0pt,
    innerbottommargin=\baselineskip,
    innerrightmargin=10pt,
    innerleftmargin=10pt,
    backgroundcolor=white!60!red}

\usepackage{amsmath}
\usepackage{amsfonts}
\usepackage{amssymb}
\usepackage{setspace}

\renewcommand{\epsilon}{\varepsilon}
\renewcommand{\phi}{\varphi}

\newcommand{\Mat}[1]{\mathbf{#1}}

\newcommand{\Prj}[2]{{
  \left.\kern-\nulldelimiterspace 
  #1 
  \vphantom{\big|} 
  \right|_{#2} 
}}

\newcommand{\defeq}{\stackrel{\mathrm{.}}{=}}

\usepackage{mathrsfs}

\newcommand{\ve}[1]{\bm{#1}}

\newtheorem{definition}{Definition}

\def\cqfd{\hfill\hbox{$\hbox{\vrule width 0.8pt
\vbox to6pt{\hrule depth 0.8pt width 5.2pt
\vfill\hrule depth 0.8pt}\vrule width 0.8pt}$}} 

\newtheorem{theorem}{Theorem}
\newtheorem{lemma}{Lemma}

\newtheorem{remark}{Remark}

\usepackage{eucal}

\usepackage{graphicx}

\usepackage{algorithm}
\usepackage{algorithmic}

\usepackage[all]{xy}
\CompileMatrices

\newcommand{\intset}[1]{\cbr{1..n}}

\newcommand{\acrotem}{TEM}

\newcommand{\innerproduct}[2]{\langle #1, #2 \rangle}

\DeclareMathOperator{\diag}{diag}

\begin{document}

\maketitle
\begin{abstract}
  In the field of optimal transport, two prominent subfields face each other: (i) unregularized optimal transport, ``\`a-la-Kantorovich'', which leads to extremely sparse plans but with algorithms that scale poorly, and (ii) entropic-regularized optimal transport, ``\`a-la-Sinkhorn-Cuturi'', which gets near-linear approximation algorithms but leads to maximally un-sparse plans. In this paper, we show that an extension of the latter to tempered exponential measures, a generalization of exponential families with indirect measure normalization, gets to a very convenient middle ground, with both very fast approximation algorithms and sparsity, which is under control up to sparsity patterns. In addition, our formulation fits naturally in the unbalanced optimal transport problem setting.
\end{abstract}

\section{Introduction}

Most loss functions used in machine learning (ML) can be related, directly or indirectly, to a comparison of positive measures (in general, probability distributions). Historically, two broad families of distortions were mainly used: $f$-divergences \cite{asAG,cEI} and Bregman divergences \cite{bTR}. Among other properties, the former are appealing because they encapsulate the notion of monotonicity of information \cite{aIG}, while the latter are convenient because they axiomatize the expectation as a maximum likelihood estimator \cite{bmdgCW}. Those properties, however, put constraints on the distributions, either on their support for the former or their analytical form for the latter.

A third class of distortion measures has progressively emerged later on, alleviating those constraints and with the appealing property to meet distance axioms: Optimal Transport distances \cite{pcCO,vOT}. Those can be interesting in wide ML fields \cite{pcCO}, but they suffer from poor scalability. A trick of balancing the metric cost with an entropic regularizer \cite{cSD} substantially improves scalability to near-optimality but blurs the frontiers with other distortion measures \cite{cSD,mnpnTR}. Most importantly, the structure of the unregularized OT plan is substantially altered through regularization: its sparsity is reduced by a factor $\Omega(n)$, $n$ being the dimension of the marginals (we consider discrete optimal transport). At the expense of an increase in complexity, getting back to a user-constrained sparse solution can be done by switching to a quadratic regularizer \cite{lpbSC}, but loses an appealing structure of the entropic-regularized OT (EOT) solution, a discrete exponential family with very specific features. Sparsity is an important topic in optimal transport: both unregularized and EOT plans are extremal in the sparsity scale, which does not necessarily fit in observed patterns \cite{pcCO}.

Finally and most importantly, optimal transport, regularized or not, does not require normalized measures; in fact, it can be extended to the unbalanced problem where marginals' total masses do not even match \cite{jmpcEO}. In that last, very general case, the problem is usually cast with approximate marginal constraints and without any constraint whatsoever on the transport plan's total mass.

In this context, our paper introduces OT on tempered exponential measures (TEMs, a generalization of exponential families), with a generalization of the EOT. Notable structural properties of the problem include training as fast as Sinkhorn balancing \textit{and} with guarantees on the solution's sparsity, also including the possibility of unbalanced optimal transport \textit{but} with tight control over total masses via their \textit{co-densities}, distributions that are used to indirectly normalize TEMs (see Figure \ref{fig:otrot}). We characterize sparsity up to sparsity patterns in the optimal solution and show that sparsity with TEMs can be interpreted as balancing the classical OT cost with an interaction term interpretable in the popular gravity model for spatial interactions \cite{hfGA}. Interestingly, this interpretation cannot hold anymore for the particular case of exponential families and thus breaks for EOT.

To maximize readability, all proofs are deferred to an appendix.

\begin{figure}[t!]
\begin{center}
    \includegraphics[trim=60bp 60bp 15bp 60bp,clip,width=0.9\linewidth]{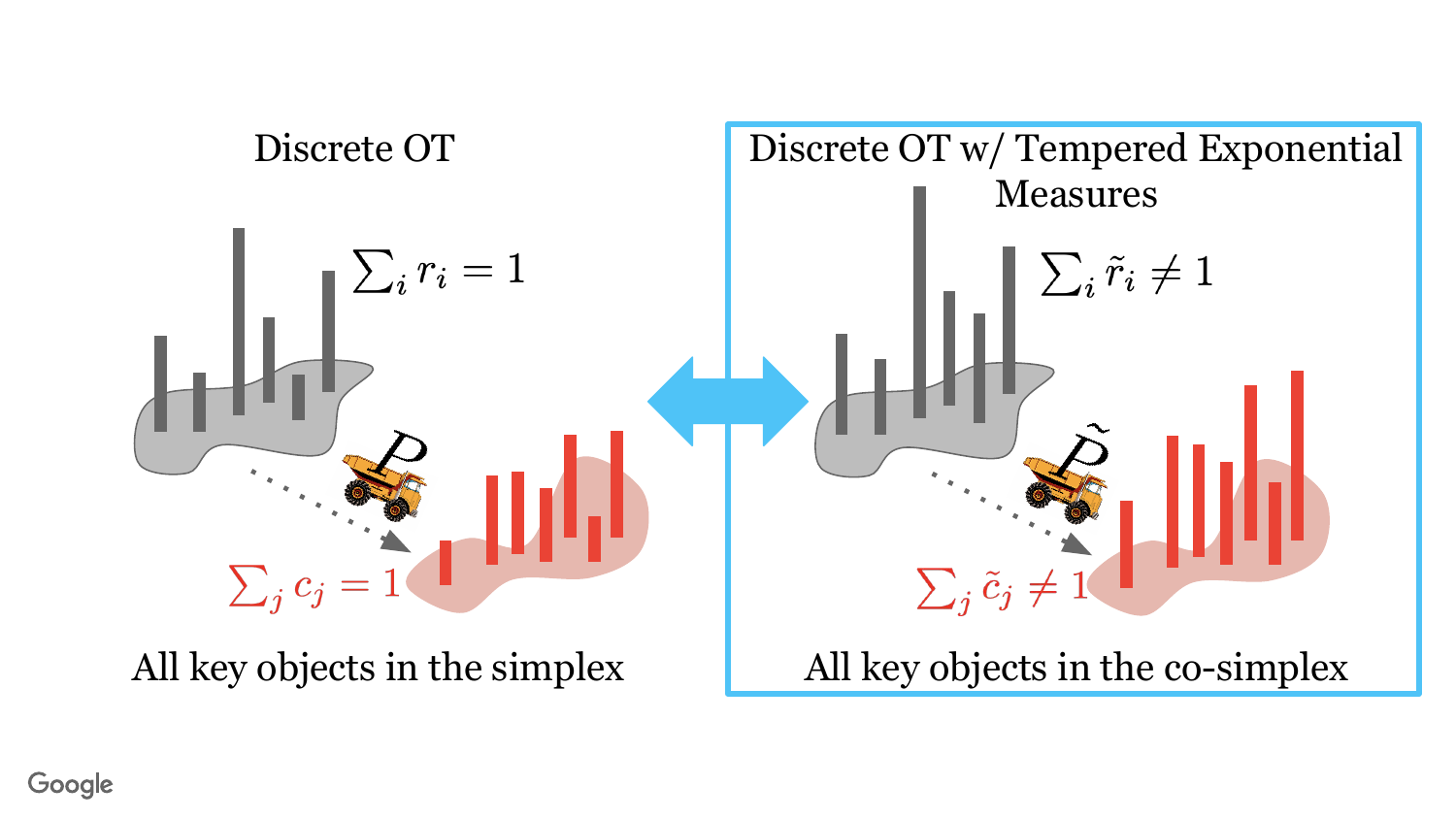}
    \caption{In classical optimal transport (OT, left), regularized or not, marginals and the OT plan sought are in the probability simplex; the optimal solution solely depends on the metric properties of the supports. Entropic regularization balances the metric cost with an entropic cost, and the optimal solution has remarkable properties related to exponential families. In this paper (right), we lift the whole setting to families of measures generalizing exponential families: tempered exponential measures (TEMs). Specific properties that appear include unbalancedness and sparsity of the optimal solution (see text).}
    \label{fig:otrot}
    \end{center}
\end{figure}

\section{Definitions}

\subsection{Optimal Transport in the Simplex} In classical discrete optimal transport (OT), we are given a cost matrix $\Mat{M} \in \mathbb{R}^{n \times n}$ ($n>1$) and two probability vectors $\ve{r}$ and column $\ve{c}$ in the simplex $\Delta_n \defeq \{\ve{p}\in \mathbb{R}^n: \ve{p} \geq \ve{0} \wedge \ve{1}^\top \ve{p} = 1\}$. Usually, $\Mat{M}$ satisfies the axioms of a distance, though only non-negativity is really important for all the results that we state. The OT problem seeks to find $d_{\Mat{M}}(\ve{r}, \ve{c}) \defeq \min_{\Mat{P} \in U_n(\ve{r}, \ve{c})} \innerproduct{\Mat{P}}{\Mat{M}} $, where $U_n(\ve{r}, \ve{c}) \defeq \{\Mat{P} \in \mathbb{R}_+^{n\times n}\vert \,\Mat{P}\ve{1}_n = \ve{r},\, \Mat{P}^\top\ve{1}_n = \ve{c}\}$ is the \emph{transport polytope} and $\innerproduct{\cdot}{\cdot}$ stands for the Frobenius dot-product. In the \emph{entropic-regularized OT} problem~\citep{cSD}, we rather seek
\begin{equation}
\label{eq:entropy-ot}
d^{\lambda}_{\Mat{M}}(\ve{r}, \ve{c}) \defeq\! \min_{\Mat{P} \in U_n(\ve{r}, \ve{c})} \innerproduct{\Mat{P}}{\Mat{M}} + \frac{1}{\lambda} \cdot \innerproduct{\Mat{P}}{\log\Mat{P}}\, , \lambda>0\,.
\end{equation}
Any discrete distribution is an exponential family \cite{aIG} but the OT plan solution to \eqref{eq:entropy-ot}, say $\Mat{P}^*$, has a special form. Denote $\ve{\mu}, \ve{\xi} \in \mathbb{R}^n$ the vectors of dual variables, corresponding to the row and column (respectively) marginalization constraints in \eqref{eq:entropy-ot}. The support of $\Mat{P}^*$ is $[n]^2$, where $[n] \defeq \{1, 2, ..., n\}$. We need to express $\Mat{P}^*$ as
\begin{eqnarray}
  P^*_{ij} = \exp(\innerproduct{\Mat{\Theta}}{\Mat{E}_{ij}} - G(\Mat{\Theta})), \label{probexp}
  \end{eqnarray}
  with $\Mat{E}_{ij}$ the matrix with general entry $(\delta_{ik}\delta_{jl})_{kl}$ (``$\delta$'' being Kronecker symbol). Let $\Mat{S} \defeq \Mat{M} - \ve{\mu}\ve{1}^\top - \ve{1}\ve{\xi}^\top$, which is a strictly positive matrix. In the unregularized case, $\Mat{S}$ encodes the slack of the constraints over the dual variables \cite[Section 2.5]{pcCO}. It follows from \citet{cSD} that the natural parameter of the exponential family is defined from those slack variables:
  \begin{eqnarray*}
    \Mat{\Theta} = -\lambda \cdot \Mat{S},
  \end{eqnarray*}
  while the cumulant or log-partition function $G$ in \eqref{probexp} is, in fact, 0 because normalization is implicitly ensured in $U_n(\ve{r}, \ve{c})$ (otherwise, the cumulant would depend on the Lagrange multiplier of the normalization constraint).

  \subsection{Tempered Exponential Measures} Any exponential family is a probability distribution that maximizes Shannon's entropy subject to a constraint on its expectation \cite{aIG}. A \textit{tempered exponential measure} (TEM) adopts a similar axiomatization \textit{but} via a generalization of Shannon's entropy (Tsallis entropy) and normalization imposed not on the TEM itself but on a so-called \textit{co-distribution} \cite{amid2023clustering}. This last constraint is a fundamental difference from previous generalizations of exponential families, $q$-exponential families, and deformed exponential families \cite{aIG}. Compared to those, TEMs also have the analytical advantage of getting a closed-form solution for the cumulant, a key ML function. A TEM has the general form (with $[z]_+ \defeq \max\{0,z\}$):
  \begin{eqnarray*}
  \tilde{p}(\ve{x}) \defeq \frac{\exp_t(\innerproduct{\ve{\theta}}{\ve{\phi}(\ve{x})})}{\exp_t(G_t(\ve{\theta}))}, \quad \exp_t (z) \defeq \left[1+(1-t) z\right]^\frac{1}{1-t}_+, \label{temexp}
  \end{eqnarray*}
  where $G_t$ is the cumulant and $\ve{\theta}$ denotes the natural parameter. The inverse of $\exp_t$ is $\log_t (z) \defeq \left(z^{1-t}-1\right)/(1-t)$, both being continuous generalizations of $\exp$ and $\log$ for $t=1$. Both functions keep their $t=1$ convexity/concavity properties for $t\geq 0$. The tilde notation above $p$ indicates that normalization does not occur on the TEM, but on a co-density defined as
  \begin{eqnarray}
    p & \defeq & \tilde{p}^{2-t} \quad \left(=\tilde{p}^{1/t^*}, \mbox{ with } t^* \defeq 1/(2-t)\right).
  \end{eqnarray}
  \begin{remark}
For a given vector $\tilde{\ve{p}}$ (or a matrix $\tilde{\Mat{P}}$) with the tilde notation, whenever convenient, we will use the convention $\ve{p} = \tilde{\ve{p}}^{1/t^*}$ (correspondingly, $\Mat{P} = \tilde{\Mat{P}}^{1/t^*}$, the exponent being coordinate-wise) whenever the tilde sign is removed.
\end{remark}
Hence, a TEM satisfies the indirect normalization $\int \tilde{p}^{2-t} \mathrm{d}\xi = \int p \mathrm{d}\xi = 1$.
\begin{remark}
  In this paper, we assume $t\in [0,1]$, though some of our results are valid for a broader range (discussed in context).
  \end{remark}
  In the same way, as KL divergence is the canonical divergence for exponential families \cite{anMO}, the same happens for a generalization in TEMs. Given two non-negative vectors $\tilde{\ve{u}}, \tilde{\ve{v}} \in \mathbb{R}^m$, we define the generalized tempered relative entropy as~\citep{Amid19}
\begin{equation*}
  D_t(\tilde{\ve{u}} \| \tilde{\ve{v}}) \defeq\!\!\! \sum_{i\in [n]}\! \tilde{u}_i\big(\log_t \tilde{u}_i  - \log_t \tilde{v}_i \big) - \log_{t-1} \tilde{u}_i + \log_{t-1} \tilde{v}_i .
\end{equation*}
Just like the KL divergence ($t\rightarrow 1$), the tempered relative entropy is a Bregman divergence, induced by the generator $\phi_t(z) \defeq z \log_t z - \log_{t-1}(z)$, which is convex for $t\in \mathbb{R}$. We also have $\phi_t'(z)=\log_t(z)$. We define the following extension of the probability simplex ${\Delta}_n$ in $\mathbb{R}^n$.
\begin{definition}\label{def:cosimplex}
The co-simplex of $\mathbb{R}^n$, $\tilde{\Delta}_n$ is defined as $\tilde{\Delta}_n \defeq \{\tilde{\ve{p}}\in \mathbb{R}^n: \tilde{\ve{p}} \geq \ve{0} \wedge \ve{1}^\top \tilde{\ve{p}}^{1/t^*} = 1\}$.
\end{definition}
Note that $\tilde{\ve{p}}^{1/t^*} \defeq \ve{p}\in \Delta_n$ iff $\tilde{\ve{p}} \in \tilde{\Delta}_n$ and $\tilde{\Delta}_n \rightarrow \Delta_n$ when $t \rightarrow 1$. Similarly, given $\tilde{\ve{r}}, \tilde{\ve{c}} \in \tilde{\Delta}_n$, we define their corresponding co-polytope in $\mathbb{R}_+^{n\times n}$.
\begin{definition}\label{def:copoly}
The co-polyhedral set of $n\times n$ non-negative matrices with co-marginals $\tilde{\ve{r}}, \tilde{\ve{c}} \in \tilde{\Delta}_n$ is defined as $\tilde{U}_n(\tilde{\ve{r}}, \tilde{\ve{c}}) \defeq \{\tilde{\Mat{P}} \in \mathbb{R}_+^{n\times n}\vert \,\tilde{\Mat{P}}^{1/t^*}\ve{1} = \tilde{\ve{r}}^{1/t^*},\, \tilde{\Mat{P}}^{1/t^*\top}\ve{1} = \tilde{\ve{c}}^{1/t^*}\}$.
\end{definition}
Likewise, $\tilde{U}_n(\tilde{\ve{r}}, \tilde{\ve{c}}) \rightarrow U_n({\ve{r}}, {\ve{c}})$ (the transport polytope) in the limit $t \rightarrow 1$. More importantly, using our notation convention,
\begin{eqnarray*}
  {\Mat{P}} \in U_n({\ve{r}}, {\ve{c}}) \mbox{\,\, iff \,\,\,} \tilde{\Mat{P}} \in \tilde{U}_n(\tilde{\ve{r}}, \tilde{\ve{c}}) .
\end{eqnarray*}
 
  \section{Related Work}

  From an ML standpoint, there are two key components to optimal transport (OT): the problem structure and its solving algorithms. While historically focused on the former \citep{mSL,kOT}, the field then became substantially ``algorithm-aware'', indirectly first via linear programming \citep{dPO} and then specifically because of its wide applicability in ML \citep{cSD}. The entropic-regularized OT (EOT) mixes metric and entropic terms in the cost function but can also be viewed as an approximation of OT in a Kullback-Leibler ball centered at the independence plan, which is, in fact, a metric \citep{cSD}. The resolution of the EOT problem can be obtained via Sinkhorn's algorithm \citep{sinkhorn1967concerning,franklin1989scaling,knight2008sinkhorn} (see Algorithm~\ref{alg:sinkhorn}), which corresponds to iterative Bregman projections onto the affine constraint sets (one for the rows and another for the columns). The algorithm requires matrix-vector multiplication and can be easily implemented in a few lines of code, making it ideal for a wide range of ML applications. However, alternative implementations of the algorithm via the dual formulation prove to be more numerically stable and better suited for high-dimensional settings \citep{pcCO}.

The structure of the solution -- the transportation plan -- is also important, and some features have become prominent in ML, like the sparsity of the solution \citep{lpbSC}. Sinkhorn iteration can be fine-tuned to lead to near-optimal complexity \citep{awrNL}, but entropic regularization suffers a substantial structural downside: the solution is maximally un-sparse, which contrasts with the sparsity of the unregularized solution \citep{pcCO}. Sparsity is a modern instantiation of ML's early constraint on the model's simplicity, otherwise known as Ockham's razor \citep{behwOR}.\footnote{\textit{Numquam ponenda est pluralitas sine necessitate}, ``plurality should never be imposed without necessity'', William of Ockham, XIV$^{th}$ century.} It has been known for a long time that ``extreme'' simplicity constraints lead to intractability for linear programming \citep{Kar72}, so sparsity in OT is desirable -- and not just for the sake of Ockham's razor \citep{bsrSA} -- but it is non-trivial and various notions of tractable sparsity can be sought, from a general objective \citep{bsrSA,mnpnTR} down to \textit{ex-ante} node specifics like transport obstruction \citep{dprRO} or limiting the transport degree \citep{lpbSC}. Sparsity makes it convenient to train from general optimizers \citep{lpbSC}. This comes, however, at the expense of losing an appealing probabilistic structure of the EOT solution, a discrete exponential family with very specific features, and eventually loses as well the near-optimal algorithmic convenience that fine-tuning Sinkhorn offers for training \citep{awrNL}.

Taking EOT as a starting point, two different directions can be sought for generalization. The first consists in replacing the entropic term with a more general one, such as Tsallis entropy (still on the simplex), which was introduced in \citet{mnpnTR}; the second consists in alleviating the condition of identical marginal masses, which is touched upon in \citet{jmpcEO} and was initially proposed without regularization by \citet{bNR}.

In work predating the focus on optimal transport \citep{permutation}, 
the same relative entropy regularized optimal transport problem was
used to develop online algorithms for learning permutations that predict close to the best permutation chosen in hindsight. See a recent result in \citet{Ballu23} on mirror Sinkhorn algorithms.

\section{Beyond Sinkhorn Distances With TEMs}

\subsection{OT Costs With TEMs} Since tempered exponential measures involve two distinct sets (the probability simplex and the co-simplex), we can naturally define two unregularized OT objectives given a cost matrix $\Mat{M} \in \mathbb{R}^{n\times n}$. The first is the classical OT cost; we denote it as the \emph{expected cost},
\begin{equation}
\label{eq:t-transport-cost-e}
    d_{\Mat{M}}^{t}(\tilde{\ve{r}}, \tilde{\ve{c}}) \defeq \min_{\tilde{\Mat{P}} \in \tilde{U}_n(\tilde{\ve{r}}, \tilde{\ve{c}})} \innerproduct{{\Mat{P}}}{\Mat{M}}\, 
  \end{equation}
(with our notations, note that the constraint is equivalent to ${\Mat{P}} \in {U}_n({\ve{r}}, {\ve{c}})$).
Instead of embedding the cost matrix on the probability simplex, we can put it directly on the co-simplex, which leads to the \emph{measured cost}:
\begin{equation}
\label{eq:t-transport-cost-m}
    \tilde{d}_{\Mat{M}}^{t}(\tilde{\ve{r}}, \tilde{\ve{c}}) \defeq \min_{\tilde{\Mat{P}} \in \tilde{U}_n(\tilde{\ve{r}}, \tilde{\ve{c}})} \innerproduct{\tilde{\Mat{P}}}{\Mat{M}}\, .
  \end{equation}
  $d_{\Mat{M}}^{t}$ is a distance if $\Mat{M}$ is a metric matrix. $\tilde{d}_{\Mat{M}}^{t}$ is trivially non-negative, symmetric, and meets the identity of indiscernibles. However, it seems to only satisfy a slightly different version of the triangle inequality, which converges to the triangle inequality as $t\rightarrow 1$.
  \begin{restatable}{proposition}{ddistone}
\label{prop:ddistone}
If $\Mat{M}$ is a distance matrix and $t\leq 1$, $(\tilde{d}^t_{\Mat{M}}(\tilde{\ve{x}}, \tilde{\ve{z}}))^{2-t} \leq M^{1-t}\cdot \left(\tilde{d}^t_{\Mat{M}}(\tilde{\ve{x}}, \tilde{\ve{y}}) + \tilde{d}^t_{\Mat{M}}(\tilde{\ve{y}}, \tilde{\ve{z}})\right)$, $\forall \tilde{\ve{x}}, \tilde{\ve{y}}, \tilde{\ve{z}} \in \tilde{\Delta}_n$, where $M \defeq \sum_{ij} M_{ij}$.
\end{restatable}
Factor $M^{1-t}$ is somehow necessary to prevent vacuity of the inequality: scaling a cost matrix by a constant $\kappa > 0$ does not change the OT optimal plan, but scales the OT cost by $\kappa$; in this case, the LHS scales by $\kappa^{2-t}$ and the RHS scales by $\kappa^{1-t}\cdot \kappa = \kappa^{2-t}$ as well. Note that it can be the case that $M<1$ so the RHS can be smaller than the triangle inequality's counterpart -- yet we would not necessarily get an inequality tighter than the triangle inequality because, in this case, it is easy to show that the LHS would also be smaller than the triangle inequality's counterpart. 
\subsection{OT Costs in a Ball} This problem is an intermediary that grounds a particular metric structure of EOT. This constrained problem seeks the optimal transport plan in an information ball -- a KL ball -- centered at the independence plan. Using our generalized tempered relative entropy, this set can be generalized as:
 \begin{equation}
\label{eq:t-alpha-poly}
    \tilde{U}_n^{\varepsilon}(\tilde{\ve{r}}, \tilde{\ve{c}}) \defeq \{\tilde{\Mat{P}} \in \tilde{U}_n(\tilde{\ve{r}}, \tilde{\ve{c}})\vert\, D_t\big(\tilde{\Mat{P}} \| \tilde{\ve{r}}\tilde{\ve{c}}^\top\big) \leq \varepsilon)\}\,,
  \end{equation}
  where $\varepsilon$ is the radius of this ball. It turns out that when $t=1$, minimizing the OT cost subject to being in this ball also yields a distance, called a Sinkhorn distance -- if, of course, $\Mat{M}$ is a metric matrix \cite{cSD}. For a more general $t$, we can first remark that 
  \begin{eqnarray*}
D_t\big(\tilde{\Mat{P}} \| \tilde{\ve{r}}\tilde{\ve{c}}^\top\big) \leq \frac{1}{1-t}, \forall \tilde{\Mat{P}} \in \tilde{\Delta}_{n\times n}, \forall \tilde{\ve{r}}, \tilde{\ve{c}} \in \tilde{\Delta}_{n}, 
  \end{eqnarray*}
  so that we can consider that
  \begin{eqnarray}
    \varepsilon & < & \frac{1}{1-t} \label{consteps}
  \end{eqnarray}
  for the ball constraint not to be vacuous. $\tilde{\ve{r}}\tilde{\ve{c}}^\top \in \tilde{U}_n(\tilde{\ve{r}}, \tilde{\ve{c}})$  is the \emph{independence table} with co-marginals $\tilde{\ve{r}}$ and $\tilde{\ve{c}}$. When $\varepsilon \rightarrow \infty$, we have $\tilde{U}^{\varepsilon}_n(\tilde{\ve{r}}, \tilde{\ve{c}}) \rightarrow \tilde{U}_n(\tilde{\ve{r}}, \tilde{\ve{c}})$. When $t\rightarrow 1$, $\tilde{U}_n^{\varepsilon}(\tilde{\ve{r}}, \tilde{\ve{c}}) \rightarrow U^{\varepsilon}_n({\ve{r}}, {\ve{c}})$, the subset of the transport polytope with bounded KL divergence to the independence table \cite{cSD}.

  Notably, the generalization of the ball $\tilde{U}_n^{\varepsilon}(\tilde{\ve{r}}, \tilde{\ve{c}})$ for $t\neq 1$ loses the convexity of the ball itself -- while the divergence $D_t$ remains convex. However, the domain keeps an important property: it is $1/t^*$-\emph{power convex}.
\begin{restatable}{proposition}{powerconvex}
\label{prop:alpha}
For any $\tilde{\Mat{P}}, \tilde{\Mat{Q}} \in \tilde{U}^{\varepsilon}_n(\tilde{\ve{r}}, \tilde{\ve{c}})$ and any $t\! \in\! \mathbb{R} \scalebox{1.2}[1.0]{-} \{2\}$,
$$(\beta\,\tilde{\Mat{P}}^{1/t^*} + (1 - \beta)\,\tilde{\Mat{Q}}^{1/t^*})^{t^*} \in \tilde{U}^{\varepsilon}_n(\tilde{\ve{r}}, \tilde{\ve{c}})\, , \forall \beta \in [0,1].$$
\end{restatable}
\subsection{Regularized OT Costs} in the case of entropic regularization, the OT cost is replaced by $\innerproduct{\Mat{P}}{\Mat{M}} + (1/\lambda) \cdot D_1\big({\Mat{P}} \| {\ve{r}}{\ve{c}}^\top\big)$. In the case of TEMs, we can formulate two types of regularized OT costs generalizing this expression, the \textit{regularized expected cost}
\begin{equation}
\label{eq:t-sinkhorn-main-e}
d^{t,\lambda}_{\Mat{M}}(\tilde{\ve{r}}, \tilde{\ve{c}}) \defeq \min_{\tilde{\Mat{P}} \in \tilde{U}_n(\tilde{\ve{r}}, \tilde{\ve{c}})} \innerproduct{\Mat{P}}{\Mat{M}} + \frac{1}{\lambda} \cdot D_t(\tilde{\Mat{P}} \| \tilde{\ve{r}}\tilde{\ve{c}}^\top)\, ,
\end{equation}
for $\lambda > 0$ and the \textit{regularized measured cost}
\begin{equation}
\label{eq:t-sinkhorn-main}
\tilde{d}^{t,\lambda}_{\Mat{M}}(\tilde{\ve{r}}, \tilde{\ve{c}}) \defeq \min_{\tilde{\Mat{P}} \in \tilde{U}_n(\tilde{\ve{r}}, \tilde{\ve{c}})} \innerproduct{\tilde{\Mat{P}}}{\Mat{M}} + \frac{1}{\lambda} \cdot D_t(\tilde{\Mat{P}} \| \tilde{\ve{r}}\tilde{\ve{c}}^\top)\, .
\end{equation}
The \textit{raison d'\^etre} of entropic regularization is the algorithmic efficiency of its approximation. As we shall see, this stands for TEMs as well. In the case of TEMs, the question remains: what is the structure of the regularized problem? In the case of EOT ($t=1$), the answer is simple, as the OT costs in a ball bring the metric foundation of regularized OT, since the regularized cost is just the Lagrangian of the OT in a ball problem. Of course, the downside is that parameter $\lambda$ in the regularized cost comes from a Lagrange multiplier, which is unknown in general, but at least a connection does exist with the metric structure of the unregularized OT problem -- assuming, again, that $\Mat{M}$ is a metric matrix.

As we highlight in Proposition \ref{prop:ddistone}, the measured cost only meets an approximate version of the triangle inequality, so a metric connection holds only in a weaker sense for $t\neq 1$. However, as we now show, when $t\neq 1$, there happens to be a \textit{direct} connection with the unregularized OT costs themselves (expected and measured), the connection to which is blurred when $t=1$ and sheds light on the algorithms we use.
\begin{restatable}{proposition}{eqentrcost}
\label{prop:eqentrcost}
For any TEM $\tilde{\Mat{P}}\in \tilde{\Delta}_{n\times n}$, any $t \in [0,1)$ and $0\leq \varepsilon \leq 1/(1-t)$, letting $\Mat{M}_t \defeq (\tilde{\ve{r}}\tilde{\ve{c}}^\top)^{1-t}$, we have
\begin{eqnarray}
  D_t\big(\tilde{\Mat{P}} \| \tilde{\ve{r}}\tilde{\ve{c}}^\top\big) \leq \varepsilon & \Leftrightarrow & \innerproduct{\tilde{\Mat{P}}}{\Mat{M}_t} \geq \exp^{1-t}_t (\varepsilon). \label{tempinner}
\end{eqnarray}
\end{restatable}
The proof is immediate once we remark that on the co-simplex, the generalized tempered relative entropy simplifies for $t\neq 1$ as:
\begin{eqnarray}
D_t\big(\tilde{\Mat{P}} \| \tilde{\ve{r}}\tilde{\ve{c}}^\top\big) & = & \frac{1}{1-t}\cdot \left(1 - \innerproduct{\tilde{\Mat{P}}}{\Mat{M}_t}\right). \label{simpldt}
  \end{eqnarray}
  Interestingly, this simplification does not happen for $t=1$, a case for which we keep Shannon's entropy in the equation and thus get an expression not as ``clean'' as \eqref{simpldt}. Though $\Mat{M}_t$ does not define a metric, it is useful to think of \eqref{tempinner} as giving an equivalence of being in the generalized tempered relative entropy ball to the independence plan (a fact relevant to information theory) and having a large OT cost with respect to a cost matrix defined from the independence plan (a fact relevant to OT). For $t\neq 1$, the constrained OT problem becomes solving one OT problem subject to a constraint on another one.
\subsection{Regularized OT Costs With TEMs Implies Sparsity} The regularized problem becomes even ``cleaner'' for the regularized measured cost \eqref{eq:t-sinkhorn-main} as it becomes an unregularized measured cost\footnote{A similar discussion, albeit more involved, holds for the expected cost. We omit it due to the lack of space.} \eqref{eq:t-transport-cost-m} over a \textit{fixed} cost matrix.
\begin{restatable}{proposition}{equivregtem}
\label{prop:equivregtem}
For any $t\in [0,1)$, the regularized measured cost \eqref{eq:t-sinkhorn-main} can be written as
  \begin{eqnarray}
    \lambda \cdot \tilde{d}^{t,\lambda}_{\Mat{M}}(\tilde{\ve{r}}, \tilde{\ve{c}}) & = & \frac{1}{1-t} + \min_{\tilde{\Mat{P}} \in \tilde{U}_n(\tilde{\ve{r}}, \tilde{\ve{c}})} \innerproduct{\tilde{\Mat{P}}}{\Mat{M}'},\label{rmceq}\\
    \Mat{M}' & \defeq & \lambda \cdot \Mat{M} - \frac{1}{1-t} \cdot \Mat{M}_t, \label{costregm}
  \end{eqnarray}
  with $\Mat{M}_t$ defined in Proposition \ref{prop:eqentrcost}.
\end{restatable}
(Proof straightforward)  This formulation shows that regularized OT with TEMs for $t\neq 1$ can achieve something that classical EOT ($t=1$) cannot: getting sparse OT plans. Indeed, as the next theorem shows, specific sparsity patterns happen in any configuration of two sources $i\neq k$ and two destinations $l\neq j$ containing two distinct paths of negative costs and at least one of positive cost.

\begin{figure}[t!]
\begin{center}
    \includegraphics[
    width=0.95\linewidth]{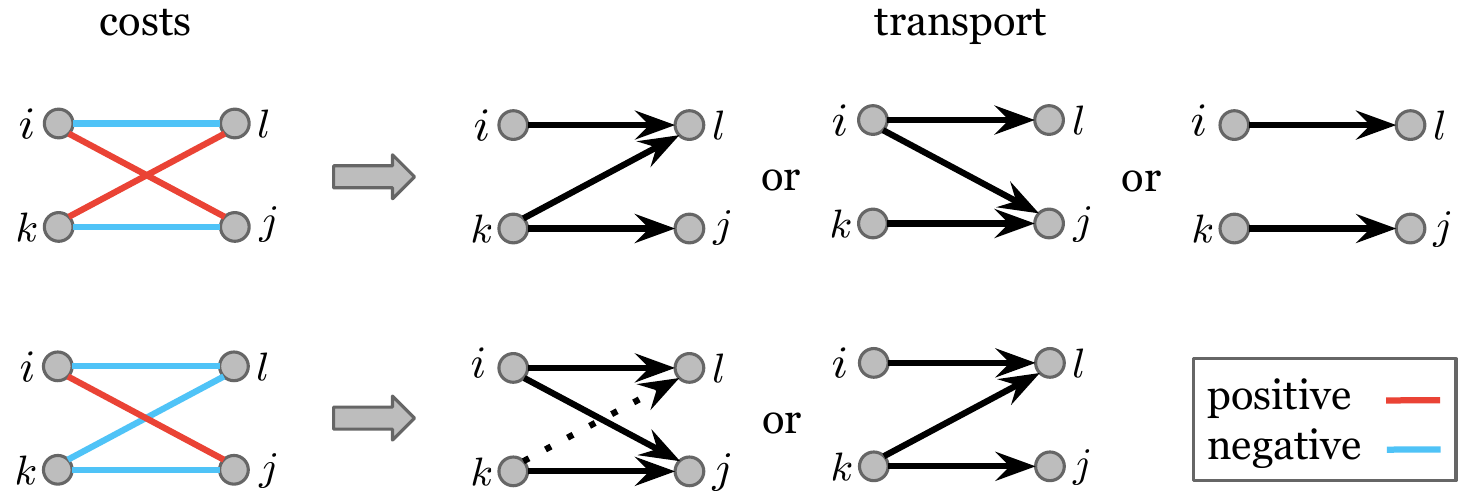}
    \caption{Illustration of Theorem \ref{th:sparseotem}. Negative costs of matrix $\Mat{M}'$ in the regularized measured cost \eqref{costregm} are in blue, and those positive are in red. The sparsity results of the theorem are shown, where no arrow means no transport and a dashed arrow means a transport necessarily ``small'' (see text).}
    \label{fig:sparse}
    \end{center}
\end{figure}

\begin{restatable}{theorem}{sparseotem}
  \label{th:sparseotem}
  Let $\tilde{\Mat{P}}$ be the optimal solution to the regularized measured cost \eqref{eq:t-sinkhorn-main}. Let $\Mat{S}$ be the support indicator matrix of $\tilde{\Mat{P}}$, defined by the general term $S_{ij} = 1$ if $\tilde{P}_{ij} > 0$ (and 0 otherwise). The following properties hold:
  \begin{enumerate}
  \item For any coordinates $(i,j)$, if $M'_{ij} < 0$ then $S_{ij} = 1$;
  \item For any coordinates $i\neq k$ and $j \neq l$, suppose we have the following configuration (Figure \ref{fig:sparse}): $M'_{ij} > 0, M'_{il} < 0, M'_{kj} < 0$. Then we have the following
    \begin{itemize}
    \item if $M'_{kl} > 0$, then $S_{ij} = 0$ or (non exclusive) $S_{kl} = 0$;
    \item if $M'_{kl} < 0$ and $S_{ij} = 1$ then necessarily
      \begin{eqnarray*}
\!\!\!\!\!\!\!\!\tilde{P}_{kl}^{1-t} \leq \frac{|M'_{kl}|}{|M'_{ij}| + |M'_{il}| + |M'_{kj}|} \cdot \max\{\tilde{P}_{ij}, \tilde{P}_{il}, \tilde{P}_{kj}\}^{1-t}.
        \end{eqnarray*}
      \end{itemize}
      \end{enumerate}
    \end{restatable}
    Theorem \ref{th:sparseotem} is illustrated in Figure \ref{fig:sparse}. Interpretations of the result in terms of transport follow: if we have a negative cost between $i,l$ and $j,k$, then ``some'' transport necessarily happens in both directions; furthermore, if the cost between $i,j$ is positive, then sparsity patterns happen:
    \begin{itemize}
    \item if the other cost $k,l$ is also positive, then we do not transport between $i,j$ or (non exclusive) $k,l$;
      \item if the other cost $k,l$ is negative, then either we do not transport between $i,j$, or the transport between $k,l$ is ``small''.
      \end{itemize}
      What is most interesting is that negative coordinates in $\Mat{M'}$ are under tight control by the user, and they enforce non-zero transport, so the flexibility of the design of $\Mat{M'}$, via tuning the strength of the regularization $\lambda$ \textit{and} the TEMs family via $t$, can allow to tightly design transport patterns. 

      \subsection{Interpretation of Matrix $\Mat{M}’$} Coordinate $(i,j)$ is $M'_{ij} = \lambda M_{ij} - (\tilde{r}_i\tilde{c}_j)^{1-t}/(1-t)$. Quite remarkably, the term $(\tilde{r}_i\tilde{c}_j)^{1-t}/(1-t)$ happens to be equivalent, up to the exponent itself, to an interaction in the gravity model if distances are constant \cite{hfGA}. If the original cost $\Mat{M}$ factors the distance, then we can get the full interaction term with the distance via its factorization. Hence, we can abstract any coordinate in $\Mat{M}'$ as:
\begin{eqnarray}
M'_{ij} & \propto & \mbox{original cost}(i,j) - \mbox{ interaction}(i,j).
\end{eqnarray}
One would then expect that OT with TEMs reflects the mixture of both terms: having a large cost wrt interaction should not encourage transport, while having a large interaction wrt cost might just encourage transport. This is, in essence, the basis of Theorem \ref{th:sparseotem}.

\begin{algorithm}[tb]
\caption{\texttt{Sinkhorn}($\Mat{K}, \ve{r}, \ve{c}$)}
\label{alg:sinkhorn}
\textbf{Input}: $\Mat{K} \in \mathbb{R}_+^{n\times n}$, $\ve{r}, \ve{c} \in \Delta_n$\\
\textbf{Output}: $\Mat{P} \in U_n(\ve{r}, \ve{c} )$
\begin{algorithmic}[1] 
\STATE $\ve{\mu} \gets \ve{1}_n$,\, $\ve{\xi} \gets \ve{1}_n$
\WHILE{not converged}
\STATE $\ve{\mu} \gets \ve{r}/(\Mat{K}\ve{\xi})$
\STATE $\ve{\xi} \gets \ve{c}/(\Mat{K}^\top\ve{\mu})$
\ENDWHILE
\STATE \textbf{return} $\diag(\ve{\mu})\,\Mat{K}\,\diag(\ve{\xi})$
\end{algorithmic}
\end{algorithm}
\begin{algorithm}[tb]
\caption{Regularized Optimal Transport with \acrotem s}
\label{alg:t-OT}
\textbf{Input}: Cost matrix $\Mat{M} \in \mathbb{R}_+^{n\times n}$, $\ve{r}, \ve{c} \in \Delta_n$, Regularizer $\lambda$\\
\textbf{Output}: $\tilde{\Mat{P}} \in \tilde{U}_n(\tilde{\ve{r}}, \tilde{\ve{c}} )$
\begin{algorithmic}[1] 
\STATE $\tilde{\Mat{K}} \gets \tilde{\Mat{K}}_e \text{\, or \,} \tilde{\Mat{K}}_m$
\STATE $\Mat{P} \gets \texttt{Sinkhorn}(\Mat{K}, \ve{r}, \ve{c})$ \,\,\, $\big(\Mat{K} = \tilde{\Mat{K}}^{1/t^*}\big)$
\STATE \textbf{return} $\Mat{P}^{t^*}$
\end{algorithmic}
\end{algorithm}

\section{Algorithms for Regularized OT With \acrotem s}

We show the analytic form of the solutions to \eqref{eq:t-sinkhorn-main-e} and \eqref{eq:t-sinkhorn-main} and explain how to solve the corresponding problems using an iterative procedure based on alternating Bregman projections. We then show the reduction of the iterative process to the Sinkhorn algorithm via a simple reparameterization.

\subsection{Regularized Expected Cost} The following theorem characterizes the form of the solution, i.e., the transport plan for the regularized expected cost OT with TEMs.
\begin{restatable}{theorem}{esol}\label{th:esol}
A solution of \eqref{eq:t-sinkhorn-main-e} can be written in the form
\begin{equation}
\label{eq:ot_sol_e_ex}
\tilde{P}_{ij} = \frac{\tilde{r}_i\tilde{c}_j}{\exp_t\left(\nu_i + \gamma_j + \lambda M_{ij}\right)}, 
\end{equation}
where $\nu_i$ and $\gamma_j$ are chosen s.t. $\tilde{\Mat{P}} \in \tilde{U}_n(\tilde{\ve{r}}, \tilde{\ve{c}})$.
\end{restatable}

\subsection{Regularized Measured Cost} The solution of the regularized measured cost OT with TEMs is characterized next.
\begin{restatable}{theorem}{msol}\label{th:msol}
A solution of \eqref{eq:t-sinkhorn-main} can be written in the form
\begin{equation}
\label{eq:ot_sol_m_ex}
\tilde{P}_{ij} =  \exp_t\left((\log_t (\tilde{r}_i\tilde{c}_j)-\lambda\, M_{ij}) \ominus_t (\nu_i + \gamma_j)\right)\, ,
\end{equation}
where $a\ominus_t b \defeq (a-b)/(1+(1-t)b)$ and $\nu_i$ and $\gamma_j$ are chosen s.t. $\tilde{\Mat{P}} \in \tilde{U}_n(\tilde{\ve{r}}, \tilde{\ve{c}})$.
\end{restatable}

\subsection{Approximation via Alternating Projections}
The Lagrange multipliers $\nu_i$ and $\gamma_j$ in the solutions~\eqref{eq:ot_sol_e_ex} and \eqref{eq:ot_sol_m_ex} no longer act as separate scaling factors for the rows and the columns because $\exp_t(a+b) \neq \exp_t(a) \cdot \exp_t(b)$ (yet, an efficient approximation is possible, \textit{Cf} below). Consequently,  the solutions are not \emph{diagonally equivalent} to their corresponding seed matrices~\eqref{eq:seed-e} and \eqref{eq:seed}. However, keeping just one marginal constraint leads to a solution that bears the analytical shape of Sinkhorn balancing.

Letting $\tilde{\Mat{P}}_{\!\circ} \in \mathbb{R}_+^{n\times n}$, the \emph{row projection} and \emph{column projection} to $\tilde{U}_n(\tilde{\ve{r}}, \tilde{\ve{c}})$ correspond to
\begin{align}
\min_{\tilde{\Mat{P}}:\, \Mat{P}\ve{1}_n = \ve{r}} & D_t(\tilde{\Mat{P}} \| \tilde{\Mat{P}}_{\!\circ})\,, \tag{row projection}\\
\min_{\tilde{\Mat{P}}:\, \Mat{P}^\top\ve{1}_n = \ve{c}} & D_t(\tilde{\Mat{P}} \| \tilde{\Mat{P}}_{\!\circ})\,, \tag{column projection}
\end{align}
in which, we use the shorthand notation $\Mat{P} = \tilde{\Mat{P}}^{1/t^*}$ (similarly for $\ve{r}$ and $\ve{c}$).
\begin{restatable}{theorem}{iters}
Given $\tilde{\Mat{P}}_{\!\circ} \in \mathbb{R}_+^{n\times n}$, the row and column projections to $\tilde{U}_n(\tilde{\ve{r}}, \tilde{\ve{c}})$ can be performed respectively via
\begin{align}
\label{eq:row-proj}
  \tilde{\Mat{P}} & = \text{\emph{diag}}\big(\tilde{\ve{r}} / \tilde{\ve{\mu}}\big)\,\tilde{\Mat{P}}_{\!\circ}\,,\;\text{ where }\, \tilde{\ve{\mu}} = \big( \tilde{\Mat{P}}^{1/t^*\top}_{\!\circ}\ve{1}_n\big)^{t^*}, \\
\label{eq:col-proj}
\tilde{\Mat{P}} & = \tilde{\Mat{P}}_{\!\circ}\,\text{\emph{diag}}\big(\tilde{\ve{c}} / \tilde{\ve{\xi}}\big)\,,\; \text{ where }\, \tilde{\ve{\xi}} = \big(\tilde{\Mat{P}}^{1/t^*\top}_{\!\circ}\ve{1}_n\big)^{t^*}.
\end{align}
\end{restatable}

It is imperative to understand the set of solutions of the alternating Bregman projections are of the form $\tilde{\Mat{P}} = \diag(\ve{\nu})\tilde{\Mat{P}}_{\!\circ} \diag(\ve{\gamma})$, whose analytical shape is different from that required by Theorems \ref{th:esol} and \ref{th:msol}. The primary reason for the solution being an approximation is the fact that the solution set by definition is non-convex for $t\neq 1$. We empirically evaluate the quality of this approximation. 

\begin{figure*}[t!]
\begin{center}
    \subfigure{\includegraphics[width=0.27\linewidth]{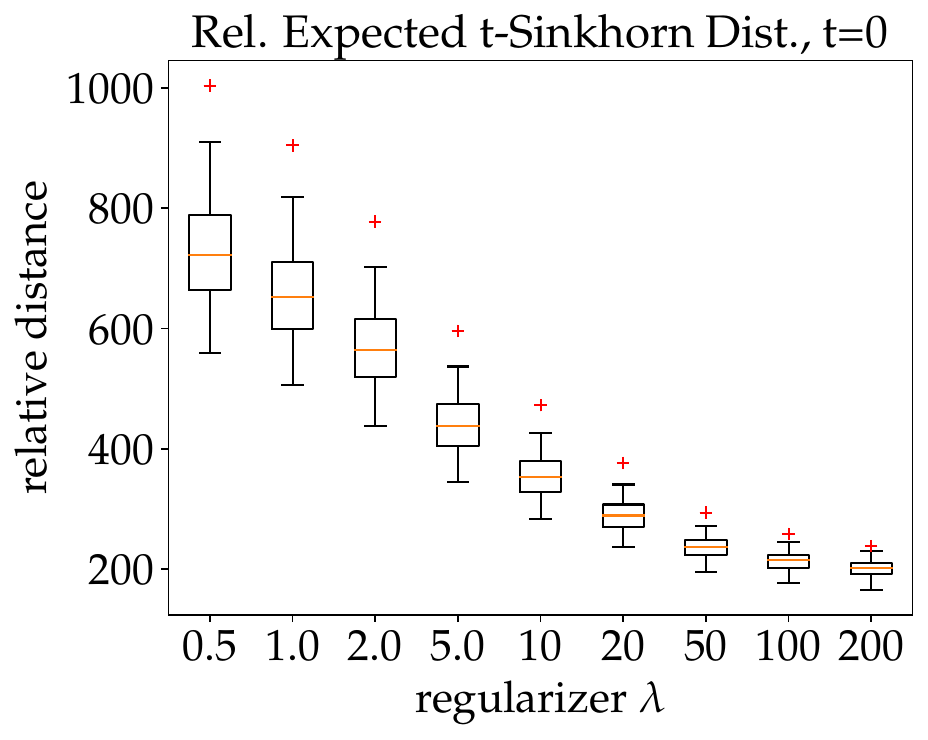}}
    \subfigure{\includegraphics[width=0.262\linewidth]{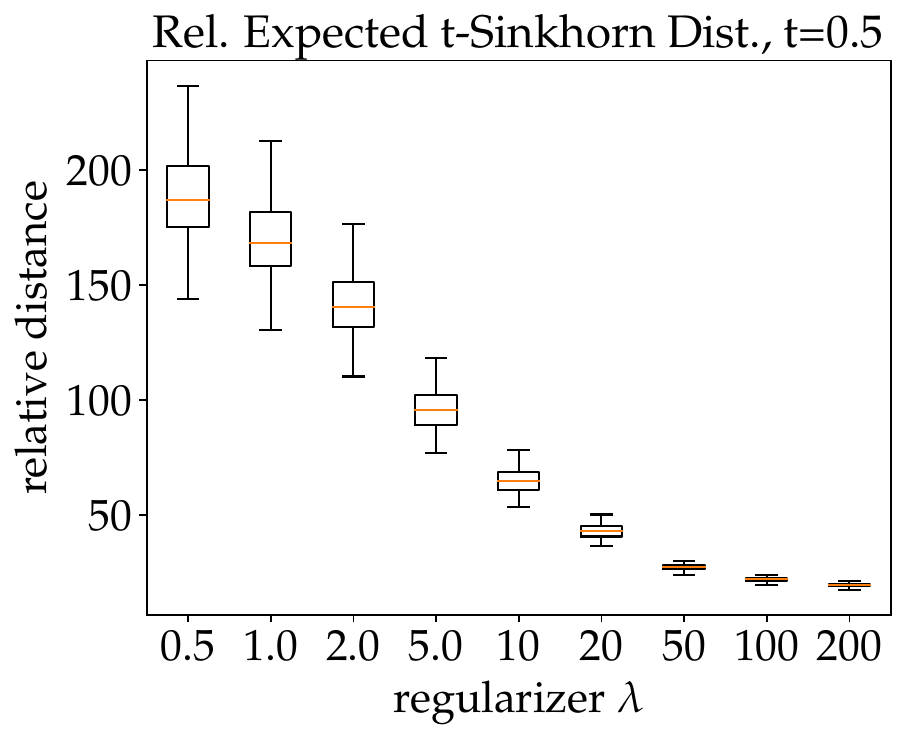}}
    \subfigure{\includegraphics[width=0.256\linewidth]{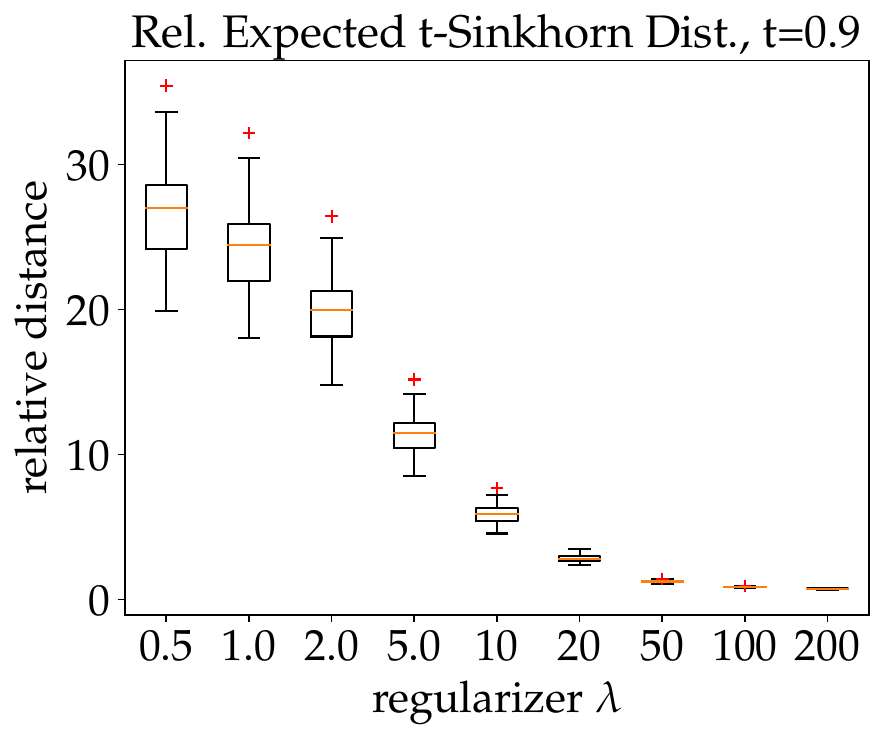}}
    \caption{The expected $t$-Sinkhorn distance relative to the Sinkhorn distance for different values of $t$. As $t\rightarrow 1$, the approximation error due to solving the unconstrained problem via alternating Bregman projections becomes smaller and the expected $t$-Sinkhorn distance converges to the Sinkhorn distance when $\lambda \rightarrow \infty$.}
    \label{fig:expected-dist}
    \end{center}
\end{figure*}
\begin{figure*}[t!]

\begin{center}
    \subfigure[]{\includegraphics[width=0.235\linewidth]{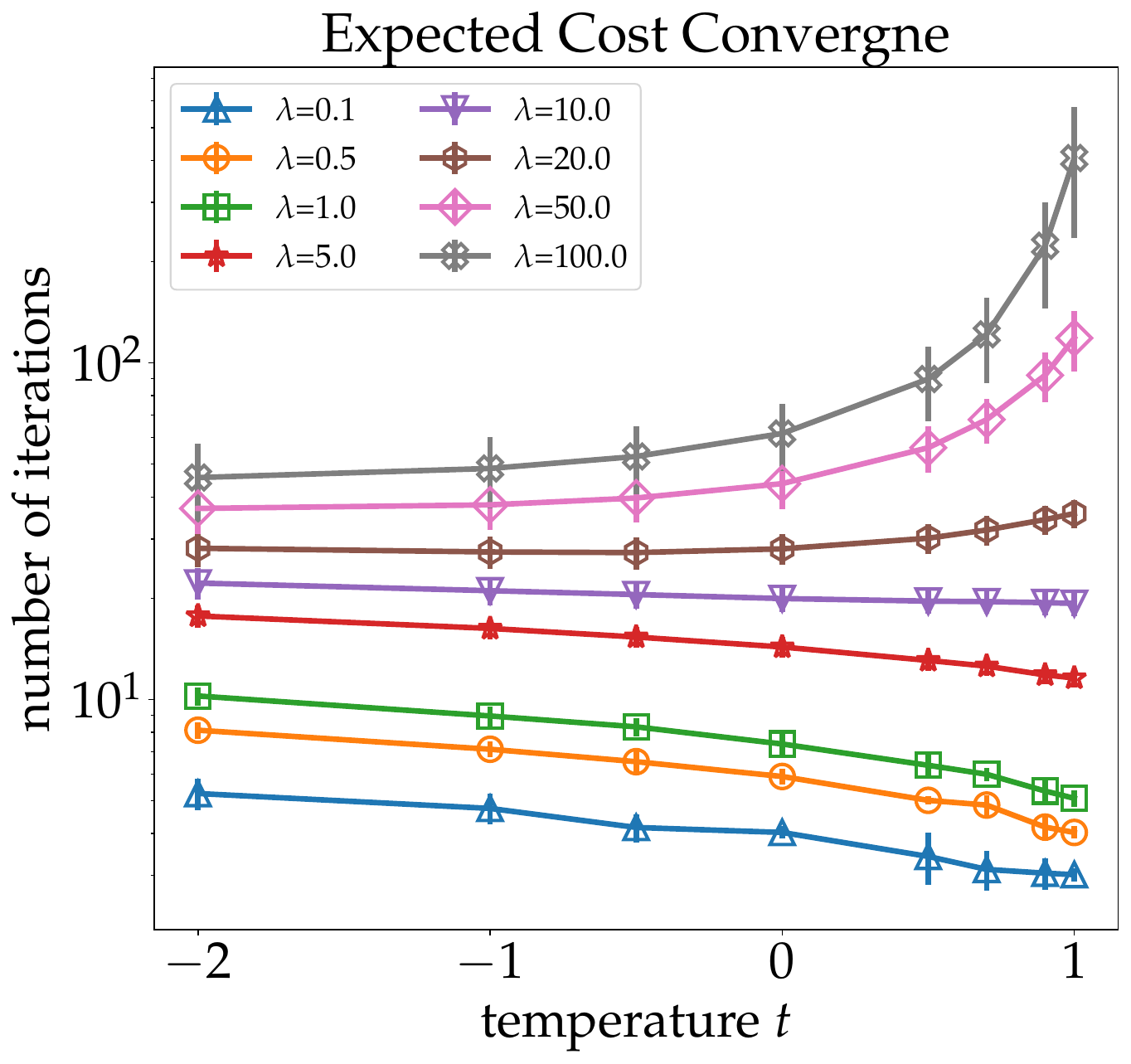}}\hspace{0.5cm}
    \subfigure[]{\includegraphics[width=0.235\linewidth]{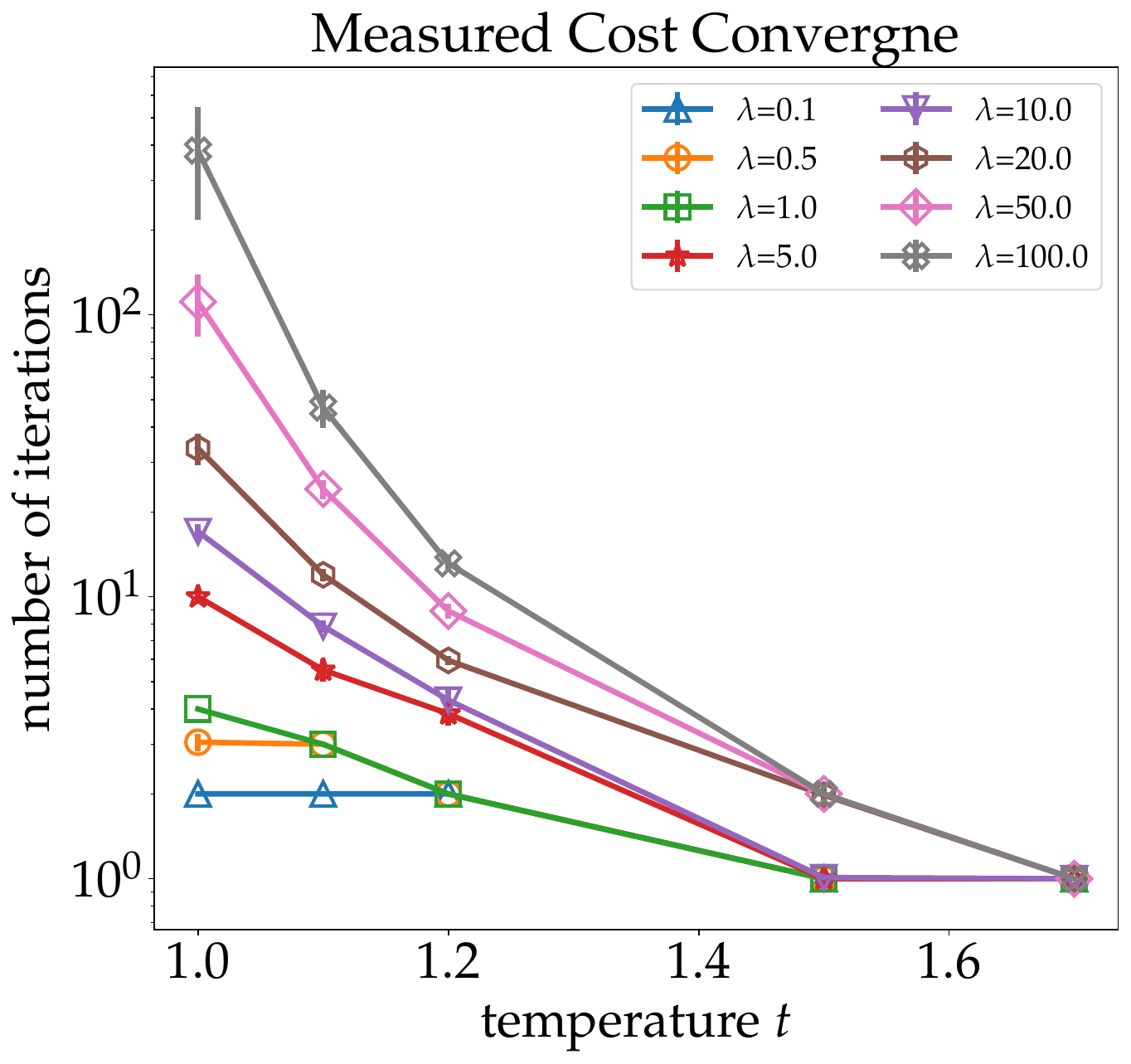}}
    \subfigure[]{\includegraphics[width=0.235\linewidth]{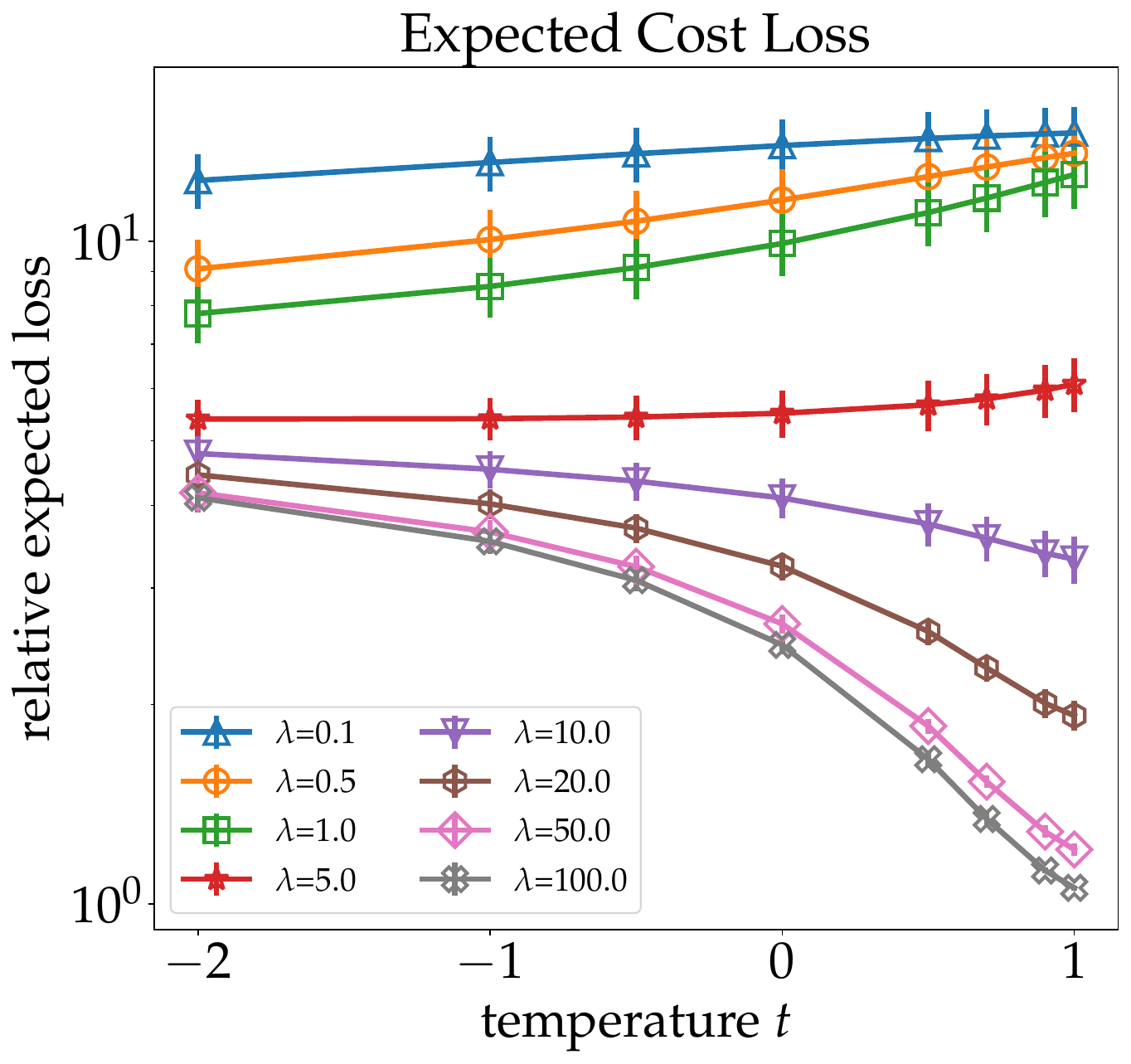}}
    \subfigure[]{\includegraphics[width=0.235\linewidth]{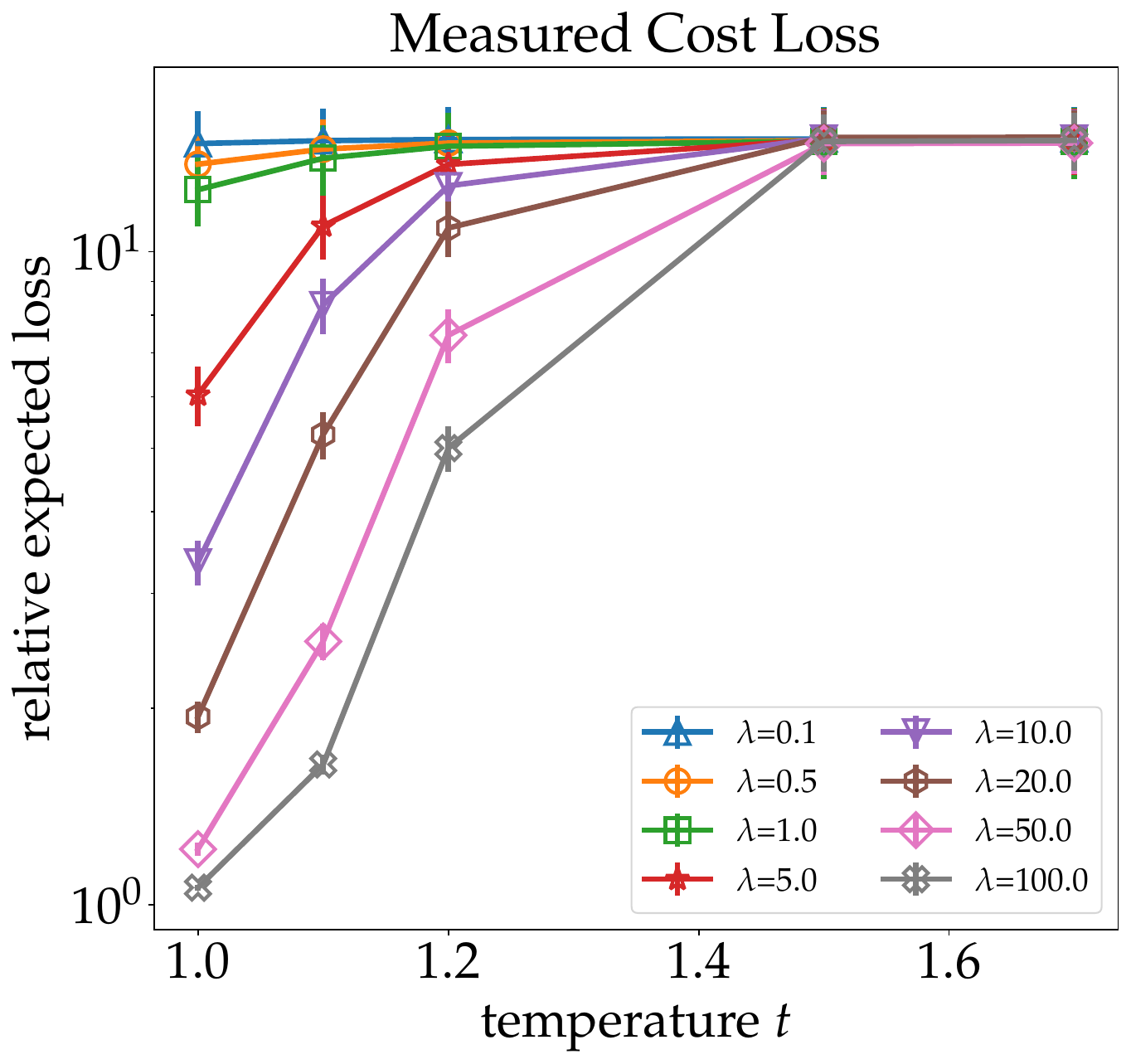}}
    \caption{Number of iterations to converge for (a) expected and (b) measured cost OT for different values of $\lambda$. Relative (to the OT solution) expected cost for the two cases, respectively, shown in (c) and (d).}
    \label{fig:convergence}
    \end{center}
\end{figure*}

\subsection{Sinkhorn Balancing for Approximating \eqref{eq:ot_sol_e_ex}\,\&\,\eqref{eq:ot_sol_m_ex}}

\subsection{Regularized Expected Cost} 
We have, in general, the possible simplification
\begin{eqnarray*}
\tilde{P}_{ij} = \frac{\tilde{r}_i\tilde{c}_j}{\exp_t(\nu_i) \otimes_t \exp_t(\sfrac{1}{t^*}\lambda M_{ij}) \otimes_t \exp_t \gamma_j}\, .
  \end{eqnarray*}
Conditions for these simplifications to be valid include $t$ close enough to $1$. We can define an \emph{expected seed matrix} for the problem~\eqref{eq:t-sinkhorn-main-e} as
\begin{equation} 
\label{eq:seed-e}
\tilde{\Mat{K}}_e \defeq  \frac{1}{\exp_t\big(\sfrac{1}{t^*}\lambda\,\Mat{M}\big)} = \exp_t\big(\ominus_t \sfrac{1}{t^*}\lambda\,\Mat{M}\big)\,,
\end{equation}
where $\ominus_t a \defeq \frac{- a}{1+(1-t) a}$ (and $\lim_{t\rightarrow 1} \ominus_t a = -a$), and the solution \eqref{eq:ot_sol_e_ex} can this time be approximated through a diagonally equivalent matrix of $\tilde{\Mat{K}}_e$ (See the preceding section with $\tilde{\Mat{P}}_{\!\circ} \defeq \tilde{\Mat{K}}_e$).

\subsection{Regularized Measured Cost} We can also simplify \eqref{eq:ot_sol_m_ex} as
\begin{equation}
\label{eq:ot_sol}
\tilde{P}_{ij} =  \frac{\exp_t\big(\log_t (\tilde{r}_i\tilde{c}_j)-\lambda M_{ij}\big)}{\exp_t(\nu_i) \otimes_t \exp_t(\gamma_j)}\, 
\end{equation}
(see the appendix) where $a \otimes_t b \defeq [a^{1-t} + b^{1-t} - 1]^{\frac{1}{1-t}}_+$ (and $\lim_{t\rightarrow 1} a \otimes_t b = a \cdot b$). The regularizer in~\eqref{eq:t-sinkhorn-main} is the Bregman divergence to the independence table, rather than the tempered entropy function $\phi_t(\tilde{\ve{P}})$; the reason being that cross terms in the numerator of the solution~\eqref{eq:ot_sol} can no longer be combined with the denominator, primarily because $\exp_t(\log_t (a) - b) \neq a\,\exp_t(-b)$ for $t \neq 1$. Furthermore, due to the normalization in~\eqref{eq:ot_sol}, the solution is not a diagonally equivalent scaling of the \emph{measured seed matrix}
\begin{equation}
\label{eq:seed}
\tilde{\Mat{K}}_m \defeq \exp_t\big(\log_t(\tilde{\ve{r}}\tilde{\ve{c}}^\top) - \lambda\,\Mat{M}\big)\,,
\end{equation}
but it can be approximated by a diagonally equivalent matrix of $\tilde{\Mat{K}}_m$ (See the preceding section with $\tilde{\Mat{P}}_{\!\circ} \defeq \tilde{\Mat{K}}_m$). In terms of sparsity patterns, the simplification \eqref{eq:ot_sol} has the direct effect of constraining the transportation plan to the coordinates $\log_t(\tilde{r}_i\tilde{c}_j) - \lambda M_{ij} > -1/(1-t)$ (otherwise, $(\tilde{K}_m)_{ij} = 0$). Remark the link with $\Mat{M}'$ in \eqref{costregm}: 
$\tilde{\Mat{K}}_m = \left[-\Mat{M}'\right]_+. $
So, the simplification \eqref{eq:ot_sol} prevents coordinates $\log_t(\tilde{r}_i\tilde{c}_j) - \lambda M_{ij}$ too small so that they are $<-1/(1-t)$ to yield non-zero transport, as would Theorem \ref{th:sparseotem} authorize. This is not an issue because, as the theorem shows, only a subset of such coordinates would yield non-zero transport anyway, and the approximation brings the benefit of being in position to design in a tighter way the desired sparsity patterns. One needs to make sure that the support of $\tilde{\Mat{K}}_m$ is big enough to allow for feasible solutions --- which is not also a real issue, granted that any optimal solution to the unregularized expected cost is so sparse that it has at most $2n-1$ non-zero values \cite{pcCO}. Both cases~\eqref{eq:seed-e} and~\eqref{eq:seed} reduce to the entropic regularized seed $\Mat{K} = \exp(-\lambda \Mat{M})$ (up to a diagonal scaling) when $t \rightarrow 1$. We then get the general approach to approximating \eqref{eq:ot_sol_e_ex} and \eqref{eq:ot_sol_m_ex}, which consists in a \textit{reduction} to Sinkhorn balancing with specific initializations.

\subsection{Solution by Reduction to Sinkhorn Balancing} 

It can be simply verified that the projection steps in~\eqref{eq:row-proj} and~\eqref{eq:col-proj} can be written in terms of the transport
polytope of $\ve{r}$ and $\ve{c}$, when working directly with the transport plan $\Mat{P} = \tilde{\Mat{P}}^{1/t^*} \in U_n(\ve{r}, \ve{c})$. Notably, the steps are identical to the standard Sinkhorn's iterations (i.e., scaling of the rows and columns), which can be computed efficiently via Algorithm~\ref{alg:sinkhorn}. The main alterations to carry out the iterations are: i) form the seed matrix $\tilde{\Mat{K}}$ via~\eqref{eq:seed-e} or~\eqref{eq:seed}, ii) apply Sinkhorn's iterations to $\tilde{\Mat{K}}^{1/t^*}$, iii) map the solution back to the co-polyhedral by computing its $t^*$-th power. See Algorithm~\ref{alg:t-OT} for the steps.


\subsection{Sparsity of Approximate Solutions}
Although the sparsity result of Theorem \ref{th:sparseotem} is for the closed-form solution of the regularized OT plan, the approximate solutions via Sinkhorn may result in a sparse solution for an appropriate choice of $t$ and for sufficiently large $\lambda$.
\begin{restatable}{proposition}{sparsity}
\label{prop:sparse}
For $\Mat{M} \in \mathbb{R}_+^{n\times n}$ (assuming $t < 2$), the expected cost seed matrix~\eqref{eq:seed-e} contains zero elements for $t > 1$ and sufficiently large  $\lambda$. Similarly, the measured cost seed matrix~\eqref{eq:seed} includes zero elements for $t < 1$ when $\lambda$ is large enough. Both matrices are positive otherwise for any $\lambda > 0$. Additionally, in both cases, for $\lambda_1 < \lambda_2$, the zero elements of the seed matrix induced by $\lambda_1$ are a subset of the zero elements induced by $\lambda_2$.
\end{restatable}
The level of the sparsity of the solution monotonically increases with $\lambda$. Nonetheless, 
when the sparsity level is too high (\emph{e.g.}, for $\lambda \rightarrow \infty$, $\tilde{\Mat{K}} \rightarrow \Mat{0}_{n\times n}$) the resulting seed matrix may no longer induce a feasible solution that is diagonally equivalent to a transport plan $\tilde{\Mat{P}} \in \tilde{U}_n(\tilde{\ve{r}}, \tilde{\ve{c}})$, as stated next.

\subsection{Convergence and Remarks on Feasibility}
\citet{franklin1989scaling} show the linear convergence of the both scaling factors $\ve{\mu}$ and $\ve{\xi}$ of Sinkhorn's algorithm for \emph{positive matrices}. Specifically, the convergence rate is proportional to the square of the \emph{contraction coefficient} $\kappa(\Mat{K}) \defeq \tanh(\sfrac{\delta(\Mat{K})}{4})$ where $\delta(\Mat{K}) \defeq \log \max_{i,j,k,\ell} \frac{K_{i\ell}K_{jk}}{K_{j\ell}K_{ik}}$ is called the \emph{projective diameter} of the linear map $\Mat{K}$. 

\begin{remark}
When the seed matrix $\tilde{\Mat{K}}$ in Algorithm~\ref{alg:t-OT} is positive, the linear convergence is then an immediate consequence of the convergence of Sinkhorn's iteration. The range of $t$ for which $\tilde{\Mat{K}}$ is a positive matrix is characterized by Proposition~\ref{prop:sparse} and the convergence rate is thus proportional to $\kappa(\tilde{\Mat{K}}^{1/t^*})^2$. Note that for $t=1$, $\kappa(\diag(\ve{r}) \exp(-\lambda\,\Mat{M})\diag(\ve{c}))\! =\! \kappa(\exp(-\lambda\,\Mat{M}))$ and both seeds~\eqref{eq:seed} and~\eqref{eq:seed-e} recover the convergence rate of the EOT.
\end{remark}
Although the convergence of Algorithm~\ref{alg:t-OT} is guaranteed for positive $\tilde{\Mat{K}}$, we still need to specify when a solution exists for non-negative $\tilde{\Mat{K}}$ (see the appendix for remarks on the feasibility). Nonetheless, \emph{if a solution exists}, we have the following result in terms of the seed and transport plan.  

\begin{restatable}{remark}{diageq}
The non-negative matrix $\tilde{\Mat{K}}$ is diagonally equivalent to $\tilde{\Mat{P}} \in \tilde{U}_n(\tilde{\ve{r}}, \tilde{\ve{c}})$ if an only if $\tilde{\Mat{K}}^{1/t^*}$ with $t^* > 0$ is diagonally equivalent to a matrix $\Mat{P} \in U_n(\ve{r}, \ve{c})$.
\end{restatable}

\section{Experiments}
We provide experimental evidence to validate the results in the paper. For each case, we sample $\Mat{M}$ uniformly between $[0, 1]$ and also sample $\ve{r}$ and $\ve{c}$ randomly. Due to limited space, we defer some of the results to the appendix.

\begin{figure*}[t!]
\begin{center}
    \includegraphics[width=0.9\linewidth]{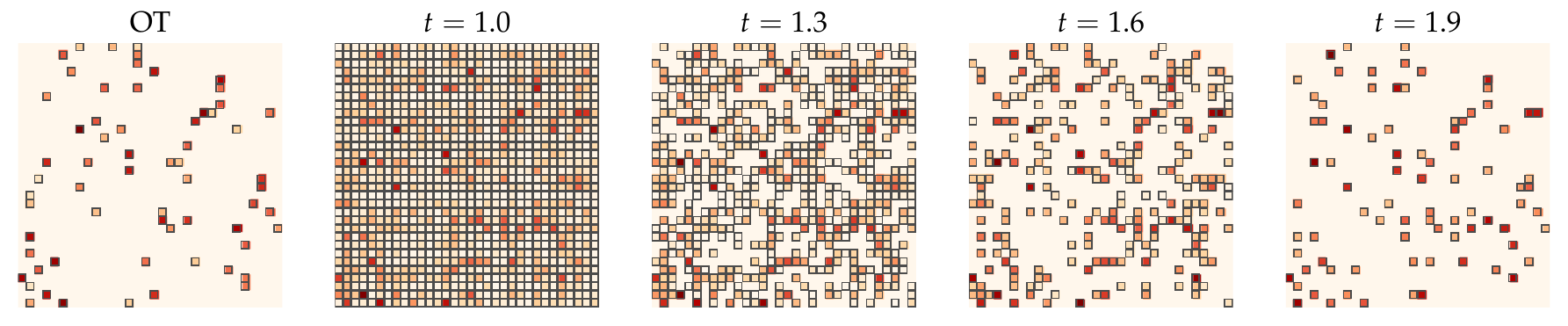}
    \caption{Transport plans induced by OT and the expected cost formulation for different values of $t$. The non-zero values are marked by a square. The EOT ($t=1$) induces a fully-dense plan. The sparsity of the solution increases by increasing $1 < t < 2$.}
    \label{fig:sparsity}
    \end{center}
\end{figure*}
\subsection{$t$-Sinkhorn Distances}
We plot the relative cost of the tempered entropic regularized OT to the value of the unregularized (measured or expected) cost. For the experiment, we set $n=64$ and average over $20$ trials. Figure~\ref{fig:expected-dist} shows the relative expected cost for different $t$ and $\lambda$. The relative cost decreases with larger $\lambda$, and the asymptotic value is closer to zero for $t$ is closer to one, which is the case for the EOT. 

\subsection{Convergence of Tempered OT}
We measure the number of steps to converge for the tempered entropic regularized OT problem using Sinkhorn's iterations for different values of $t$ and $\lambda$. We stop Sinkhorn's iterations when the maximum absolute change in each coordinate of $\ve{\xi}$ is less than $1\mathrm{e}{-10}$. For the experiment, we set $n=64$ and average over $100$ trials. Figure~\ref{fig:convergence} shows the number of iterations to converge along with the relative \emph{expected} cost to the solution of the unregularized OT problem. The number of iterations to converge follows a similar pattern to the contraction ratios of the seed matrices, shown in Figure~\ref{fig:contraction} in the appendix, while the relative expected cost is inversely proportional to the number of iterations. This result highlights the trade-off between the convergence and the (expected) transport cost.

\subsection{Sparse Solutions}
\label{subsec:sparsity}
We analyze the sparsity of the solution of the (unregularized) OT problem as well as the solutions of the regularized expected cost problem~\eqref{eq:t-sinkhorn-main-e} for $t \in [1, 2)$. Note that $t=1$ is equal to the EOT problem~\eqref{eq:entropy-ot}. We set $n=32$ and, for each case, set $\lambda = 6.0 / t^*$ to offset the scaling factor in~\eqref{eq:seed-e}. In Figure~\ref{fig:sparsity}, we show the non-zero values of the transport plans (more precisely, values larger than $1\mathrm{e}{-25}$). OT induces a sparse solution with $2n - 1 = 63$ non-zero components. On the other hand, the EOT ($t=1$) solution is fully dense, with 1024 non-zero components. The sparsity increased by increasing $t \in [1, 2)$. In this case, the transport plan with $t = 1.9$ has only $83$ non-zero values. More results for the regularized measured cost are given in the appendix.

\section{Conclusions}
We investigated the regularized version of the
optimal transport problem with tempered exponential measures. The regularizations are Bregman divergences induced by the negative tempered Tsallis entropy. We studied how regularization affects the sparsity pattern of the solution and adapted Sinkhorn balancing to quickly approximate the solution.
\section*{Acknowledgments}

The authors warmly thank Mathieu Blondel for remarks and discussions around the material presented.

\bibliography{refs}

\clearpage
\newpage
\appendix
\onecolumn
\begin{center}
\huge{Optimal Transport with Tempered Exponential Measures}\\
\huge{(Appendix)}
\end{center}

\section{A Primer on Tempered Algebra and Tempered Exponential Measures (TEMs)}
\label{sec-sup-primer}

\subsection{Tempered Algebra} We make use of two operations introduced in \citet{nlwGA} and generalizing their counterpart over the reals for $t=1$,
\begin{itemize}
\item the tempered subtraction: $a \ominus_t b \defeq (a-b)/(1+(1-t)b)$, defined for $(a,b) \in \mathbb{R}\times \mathbb{R}_{\neq -1/(1-t)}$;
  \item the tempered product: $a \otimes_t b \defeq \left[a^{1-t} + b^{1-t} - 1\right]_+^{\frac{1}{1-t}}$, defined for $(a,b) \in \mathbb{R}^2_{\geq 0}$.
  \end{itemize}
  Under a wide range of their parameters, they can be used to simplify tempered exponential expressions, generalizing the $t=1$ case.
\begin{lemma}\label{lemominus}
  Let $t\in [0,1)$. Then
  \begin{eqnarray}
\exp_t(u \ominus_t v) = \frac{\exp_t(u)}{\exp_t(v)} , \forall u \in \mathbb{R}, \forall v \geq -\frac{1}{1-t} \label{propominus}
  \end{eqnarray}
  (otherwise, only the LHS is defined).
\end{lemma}
\begin{lemma}\label{lemotimes}
  Let $t\in [0,1)$. Then
  \begin{eqnarray}
\exp_t(u + v) = \exp_t(u) \otimes_t \exp_t(v) & \mbox{ iff } & \left(u\leq 0 \wedge v\leq 0\right)\vee\left(u \geq -\frac{1}{1-t} \wedge v \geq -\frac{1}{1-t}\right) \label{propotimes}
  \end{eqnarray}
  (otherwise, $\exp_t(u + v) < \exp_t(u) \otimes_t \exp_t(v)$).
\end{lemma}
\begin{proof}
  The LHS of \eqref{propotimes} has the truth value of the predicate
  \begin{eqnarray}
\left[1 + (1-t)(u+v)\right]_+ & = & \left[ \left[1 + (1-t)u\right]_+ +\left[1 + (1-t)v\right]_+ -1\right]_+ .\label{simplotimes}
  \end{eqnarray}
  \noindent\textbf{Case 1} Suppose first $u\leq 0 \wedge v\leq 0$ so that we can write $u = -\alpha / (1-t)$ and $v = -\beta / (1-t)$ for some $\alpha, \beta \geq 0$. If $\alpha + \beta \geq 1$, the LHS of \eqref{simplotimes} is 0. The RHS is $\left[ \left[1-\alpha\right]_+ +\left[1-\beta\right]_+ -1\right]_+ \leq \left[ \left[1-\alpha\right]_+ +\left[1-\left[1-\alpha\right]_+\right]_+ -1\right]_+ = \left[ \left[1-\alpha\right]_+ + \min\{1,\alpha\} - 1\right]_+ = 0$ and thus the RHS is equal to 0 because it cannot be negative, and \eqref{simplotimes} is true. If however $\alpha + \beta < 1$, the LHS of \eqref{simplotimes} is $1 - (\alpha + \beta)$. The RHS is trivially $\left[ \left[1-\alpha\right]_+ +\left[1-\beta\right]_+ -1\right]_+ = \left[ 1-\alpha + 1-\beta  -1\right]_+ = 1 - (\alpha + \beta)$, and \eqref{simplotimes} is true.\\
  \noindent\textbf{Case 2} Suppose now $u \geq -\frac{1}{1-t} \wedge v \geq -\frac{1}{1-t}$. The RHS of \eqref{simplotimes} is $\left[ \left(1 + (1-t)u\right) +\left(1 + (1-t)v\right) -1\right]_+ = \left[ 1 + (1-t)(u+v)\right]_+$, equal to the LHS.\\
  Cases 1 and two show $\Leftarrow$ in \eqref{propotimes}. To show $\Rightarrow$, we proceed by contraposition. The negation of the RHS in \eqref{propotimes} is equivalent to a single case (the other follows by symmetry): $u > 0 \wedge v < -1/(1-t)$. In this case, the RHS of \eqref{simplotimes} is $=(1-t) u$, while the LHS is $< \left[1 + (1-t)u - 1\right]_+ = (1-t)u$, thus contradicting the LHS in \eqref{propotimes}.
  \end{proof}

\subsection{Basics of TEMs} We describe here the minimal amount of material necessary to understand how our approach to boosting connects to these measures. We refer to \citet{amid2023clustering} for more details. With a slight abuse of notation, we define the perspective transforms $(\log_t)^*(z) \defeq t^* \cdot \log_{t^*} (z/t^*)$ and $(\exp_t)^*(z) \defeq t^* \cdot \exp_{t^*} (z/t^*)$. Recall that $t^* \defeq 1/(2-t)$.
\begin{definition}\citep{amid2023clustering}
  A tempered exponential measure (\acrotem) family is a set of unnormalized densities in which each element admits the following canonical expression:
  \begin{eqnarray}
\label{eq:exp_t_density_form}
q_{t|\ve{\theta}} (\ve{x}) \defeq \frac{\exp_t(\ve{\theta}^\top \ve{\phi}(\ve{x}))}{\exp_t(G_t(\ve{\theta}))} = \exp_t(\ve{\theta}^\top \ve{\phi}(\ve{x}) \ominus_t G_t(\ve{\theta})),
  \end{eqnarray}
  where $\ve{\theta}$ is the element's natural parameter, $\ve{\phi}(\ve{x})$ is the sufficient statistics and
  \begin{eqnarray*}
    G_t(\ve{\theta}) & = & (\log_t)^* \int (\exp_t)^* (\ve{\theta}^\top \ve{\phi}(\ve{x}))\mathrm{d}\xi
  \end{eqnarray*}
  is the (convex) cumulant, $\xi$ being a base measure (implicit).
\end{definition}
We remark that the simplification \eqref{eq:exp_t_density_form} indeed holds from \ref{lemominus} because it follows that $G_t(\ve{\theta}) \geq -1/(1-t)$ from its expression. Except for $t=1$ (which reduces a \acrotem~family to a classical exponential family), the total mass of a \acrotem~is not 1 (but it has an elegant closed form expression \cite{amid2023clustering}). However, the exponentiated $q_{t|\ve{\theta}}^{1/t^*}$ does sum to 1. In the discrete case, this justifies extending the classical simplex to what we denote as the co-simplex.

\section{Proofs}
\label{sec-sup-pro}

\ddistone*
\begin{proof}
Let $M \defeq \sum_{ij} M_{ij}$. Let $\tilde{\Mat{P}} \in \tilde{U}_n(\tilde{\ve{x}}, \tilde{\ve{y}})$ and $\tilde{\Mat{Q}} \in \tilde{U}_n(\tilde{\ve{y}}, \tilde{\ve{z}})$ are two optimal solutions for $\tilde{d}^t_{\Mat{M}}(\tilde{\ve{x}}, \tilde{\ve{y}})$ and $\tilde{d}^t_{\Mat{M}}(\tilde{\ve{y}}, \tilde{\ve{z}})$, respectively. Then, $\tilde{\Mat{S}} \in \tilde{U}_n(\tilde{\ve{x}}, \tilde{\ve{z}})$ where $\tilde{S}_{ik} \defeq \big(\sum_j \frac{\tilde{P}_{ij}^{1/t^*}\tilde{Q}_{jk}^{1/t^*}}{\tilde{y}^{1/t^*}_j}\big)^{t^*} = \big(\sum_j \frac{{P}_{ij}{Q}_{jk}}{{y}_j}\big)^{t^*}$. We have
\begin{align*}
  \tilde{d}^t_{\Mat{M}}(\tilde{\ve{x}}, \tilde{\ve{z}}) \leq \innerproduct{\tilde{\Mat{S}}}{\Mat{M}} & = \sum_{ik} M_{ik} \tilde{S}_{ik}\\
                                                                                                                  & = M \cdot \sum_{ik} \frac{M_{ik}}{M} \left(\sum_j \frac{P_{ij}Q_{jk}}{{y}_j}\right)^{t^*}\\
                                                                                                                  & \leq M \cdot\left( \sum_{ik} \frac{M_{ik}}{M} \cdot \sum_j \frac{P_{ij}Q_{jk}}{{y}_j}\right)^{t^*}\\
                                                                                                                  & = M^{1-t^*} \cdot \left( \sum_{ijk} M_{ik} \frac{P_{ij}Q_{jk}}{{y}_j}\right)^{t^*}\\
                                                                                                                  & \leq M^{1-t^*} \cdot \left( \sum_{ijk} (M_{ij} + M_{jk}) \frac{P_{ij}Q_{jk}}{{y}_j}\right)^{t^*}\\
  & = M^{1-t^*} \cdot \left( \sum_{ij} M_{ij}P_{ij} + \sum_{jk} M_{jk}Q_{jk} \right)^{t^*}\\
  & \leq M^{1-t^*} \cdot \left( \sum_{ij} M_{ij}\tilde{P}_{ij} + \sum_{jk} M_{jk}\tilde{Q}_{jk} \right)^{t^*}\\
    & = M^{1-t^*} \cdot \left(\tilde{d}^t_{\Mat{M}}(\tilde{\ve{x}}, \tilde{\ve{y}}) + \tilde{d}^t_{\Mat{M}}(\tilde{\ve{y}}, \tilde{\ve{z}})\right)^{t^*}\,,
\end{align*}
resulting in $(\tilde{d}^t_{\Mat{M}}(\tilde{\ve{x}}, \tilde{\ve{z}}))^{2-t} \leq M^{1-t}\cdot \left(\tilde{d}^t_{\Mat{M}}(\tilde{\ve{x}}, \tilde{\ve{y}}) + \tilde{d}^t_{\Mat{M}}(\tilde{\ve{y}}, \tilde{\ve{z}})\right)$. The first inequality is Jensen's inequality on $z\rightarrow z^{t^*}$ (concave for $t\leq 1$), and the second uses the fact that $\Mat{M}$ is a metric matrix.
The last inequality comes from $\innerproduct{{\Mat{R}}}{\Mat{M}} \leq \innerproduct{\tilde{\Mat{R}}}{\Mat{M}}$ for any TEM matrix $\tilde{\Mat{R}} \in \tilde{U}_n(\tilde{\ve{r}}, \tilde{\ve{c}})$ because $\tilde{R}_{ij} = R_{ij}^{t^*} \geq R_{ij}$ since $R_{ij} \leq 1$, for any $t\leq 1$.
  \end{proof}

\powerconvex*
\begin{proof}
Given $\tilde{\Mat{P}}, \tilde{\Mat{Q}}\in \tilde{U}^{\varepsilon}_n(\tilde{\ve{r}}, \tilde{\ve{c}})$, we have
\begin{equation}
\begin{split}
\label{eq:PQ_const}
    \sum_{ij} & (\tilde{P}_{ij} \log_t \tilde{P}_{ij} - \tilde{P}_{ij} \log_t(\tilde{r}_i \tilde{c}_j) = \frac{1}{1-t} (1 - \sum_{ij} \tilde{P}_{ij} (\tilde{r}_i \tilde{c}_j)^{1 - t}) \leq \varepsilon
\end{split}
\end{equation}
(similarly for $\tilde{\Mat{Q}}$), which implies 
\begin{equation}
    \label{eq:PQ_convex_const}
    \frac{1}{1-t} (1 - \sum_{ij} (\beta\, \tilde{P}_{ij} + (1 - \beta)\, \tilde{Q}_{ij})\, (\tilde{r}_i \tilde{c}_j)^{1 - t}) \leq \varepsilon\,,
\end{equation}
for $\beta \in [0, 1]$. For all $i, j \in [m]$, 
by the weighted power mean inequality we have that
\begin{equation}
\label{eq:PQ_t_less}
    \beta\,\tilde{P}_{ij} + (1 - \beta)\,\tilde{Q}_{ij} \leq (\beta\,\tilde{P}_{ij}^{1/t^*} + (1 - \beta)\,\tilde{Q}_{ij}^{1/t^*})^{t^*}\,,
\end{equation}
if $t \leq 1$ and 
\begin{equation}
\label{eq:PQ_t_greater} 
    (\beta\,\tilde{P}_{ij}^{1/t^*} + (1 - \beta)\,\tilde{Q}_{ij}^{1/t^*})^{t^*} \leq  \beta\,\tilde{P}_{ij} + (1 - \beta)\,\tilde{Q}_{ij}\,,
\end{equation}
if $t \geq 1$. Applying the appropriate inequality from~\eqref{eq:PQ_t_less} and \eqref{eq:PQ_t_greater} to \eqref{eq:PQ_convex_const} and noticing the sign of $\frac{1}{1 - t}$ concludes the proof.
\end{proof}

\msol*
\begin{proof}
Using Lagrange multipliers $\ve{\nu}', \ve{\gamma}' \in \mathbb{R}^n$ and $\Mat{\Pi} \in \mathbb{R}_+^{n\times n}$, we can write the problem as
\begin{eqnarray}
  \min_{\tilde{\Mat{P}} \in \mathbb{R}^{n\times n}}\Big\{ \innerproduct{\tilde{\Mat{P}}}{\Mat{M}} + \frac{1}{\lambda} \cdot D_t(\tilde{\Mat{P}}\| \tilde{\ve{r}}\tilde{\ve{c}}^\top)  +  {\ve{\nu}'}^\top(\tilde{\Mat{P}}^{1/t^*}\ve{1}_n - \tilde{r}^{1/t^*}) + {\ve{\gamma}'}^\top(\tilde{\Mat{P}}^{\top1/t^*}\ve{1}_n - \tilde{c}^{1/t^*}) - \innerproduct{\Mat{\Pi}}{\tilde{\Mat{P}}}\Big\}\, . \label{pmin}
\end{eqnarray}
Setting the derivative to zero in terms of $\tilde{P}_{ij}$, we have
\[
M_{ij} + \frac{1}{\lambda} (\log_t \tilde{P}_{ij} -\log_t (\tilde{r}_i\tilde{c}_j))+ (2-t)\nu'_i\, \tilde{P}_{ij}^{1-t} + (2-t)\gamma'_j\,\tilde{P}_{ij}^{1-t} - \Pi_{ij} = 0, \forall i,j\, .
\]
Applying the definition of $\log_t$ and rearranging the terms, we have
\begin{align}
\tilde{P}_{ij}^{1-t}\, \big(1 + (1-t)(2-t)\lambda\, (\nu'_i + \gamma'_j)\big) & = 1 + (1 - t)\,(\log_t (\tilde{r}_i\tilde{c}_j)-\lambda\, M_{ij}) + \Pi_{ij} \cdot (1-t)\lambda .\label{pequal}
\end{align}
Note we can assume wlog that $1 + (1-t)(2-t)\lambda\, (\nu'_i + \gamma'_j) \neq 0, \forall i,j$. If it were not the case, the dependency of \eqref{pmin} in $\tilde{P}_{ij}$ vanishes, so we could solve it without $\tilde{P}_{ij}$. This yields
\begin{eqnarray*}
  \tilde{P}_{ij}^{1-t} & = & \frac{1 + (1 - t)\Big(\log_t (\tilde{r}_i\tilde{c}_j)-\lambda\, M_{ij}\Big) + \Pi_{ij} \cdot (1-t)\lambda}{1 + (1-t)(2-t)\lambda\, (\nu'_i + \gamma'_j)} \\
  & = & 1 + (1-t) \cdot \left[(\log_t (\tilde{r}_i\tilde{c}_j)-\lambda\, M_{ij}) \ominus_t (2-t)\lambda\, (\nu'_i + \gamma'_j)\right] + \Pi_{ij} \cdot \frac{(1-t) \lambda}{1 + (1-t)(2-t)\lambda\, (\nu'_i + \gamma'_j)} .
  \end{eqnarray*}
  From the KKT conditions, if $\Pi_{ij} > 0$, then $\tilde{P}_{ij} = 0$. Hence, if $1 + (1-t) \cdot \left[(\log_t (\tilde{r}_i\tilde{c}_j)-\lambda\, M_{ij}) \ominus_t (2-t)\lambda\, (\nu'_i + \gamma'_j)\right] < 0$, then $\tilde{P}_{ij} = 0$, implying
  \begin{eqnarray*}
    \tilde{P}_{ij} & = & \left[1 + (1-t) \cdot \left[(\log_t (\tilde{r}_i\tilde{c}_j)-\lambda\, M_{ij}) \ominus_t (2-t)\lambda\, (\nu'_i + \gamma'_j)\right] \right]_+^{\frac{1}{1-t}}\\
                         & = & \exp_t\left((\log_t (\tilde{r}_i\tilde{c}_j)-\lambda\, M_{ij}) \ominus_t (\nu_i + \gamma_j)\right).
  \end{eqnarray*}
  with $\nu_i \defeq (2-t)\lambda\nu'_i$, $\gamma_j \defeq (2-t)\lambda\gamma'_j$. 
\end{proof}

\sparseotem*
\begin{proof}
  We show Theorem \ref{th:sparseotem} after Theorem \ref{th:msol} because its proof builds upon the proof of the theorem. We proceed in four steps, each of them building only on the KKT conditions, optimization being over a non-convex set.\\
    \noindent \textbf{Step 1} analytical form of $\tilde{\Mat{P}}$. Using the simplification \eqref{simpldt} in the derivative of the Lagrangian yields a formulation of $\tilde{P}_{ij}$ which is more readable for our purpose:
  \begin{eqnarray}
\tilde{P}_{ij}^{1-t} & = & \frac{\Pi_{ij} - M'_{ij}}{(2-t)(\nu'_i + \gamma'_j)}.\label{eqptilde}
  \end{eqnarray}
  \noindent \textbf{Step 2} Denoting the support $\mathcal{S}$ of the optimal TEM $\tilde{\Mat{P}}$ as $\mathcal{S} \defeq \{(i,j) : \tilde{P}_{ij} > 0\}$, we show
  \begin{eqnarray}
\mathcal{S} \supseteq \{(i,j) : M'_{ij} < 0\}.
  \end{eqnarray}
  In words, all coordinates with negative cost are in the support of $\tilde{\Mat{P}}$. This is, in fact, a direct consequence of \eqref{eqptilde}: since KKT conditions impose $\Pi_{ij} \geq 0$, if $M'_{ij} < 0$, the numerator is strictly positive and as a consequence $\tilde{P}_{ij}>0$ and $\Pi_{ij} = 0$.\\
  \noindent \textbf{Step 3} if $M_{kl} > 0$ and $\tilde{P}_{ij} > 0$, then $\tilde{P}_{kl} = 0$. Suppose otherwise: $M_{kl} > 0 \wedge \tilde{P}_{ij} > 0 \wedge \tilde{P}_{kl} > 0$. After Step 2, we thus have all four $\tilde{P}_{ij}$, $\tilde{P}_{il}$, $\tilde{P}_{kj}$, $\tilde{P}_{kl}$ strictly positive. As a consequence all four Lagrange multipliers $\Pi_{ij}=\Pi_{il}=\Pi_{kj}=\Pi_{kl} = 0$. Hence,
  \begin{itemize}
  \item Since $M'_{ij} > 0$, \eqref{eqptilde} yields $\nu'_i + \gamma'_j < 0$;
  \item Since $M'_{il} < 0$, \eqref{eqptilde} yields $\nu'_i + \gamma'_l > 0$;
  \item Since $M'_{kj} < 0$, \eqref{eqptilde} yields $\nu'_k + \gamma'_j > 0$;
    \item Since $M'_{kl} > 0$, \eqref{eqptilde} yields $\nu'_k + \gamma'_l < 0$.
    \end{itemize}
    Putting these in order, we get
    \begin{eqnarray*}
\nu'_i < -\gamma'_j < \nu'_k < -\gamma'_l < \nu'_i,
    \end{eqnarray*}
    a contradiction.\\
    \noindent \textbf{Step 4} if $M_{kl} < 0$ and $\tilde{P}_{ij} > 0$ then
\begin{eqnarray*}
\tilde{P}_{kl}^{1-t} & \leq & \frac{|M'_{kl}|}{|M'_{ij}| + |M'_{il}| + |M'_{kj}|} \cdot \max\{\tilde{P}_{ij}, \tilde{P}_{il}, \tilde{P}_{kj}\}^{1-t},
      \end{eqnarray*}
Suppose otherwise, so that the inequality holds with ``$>$''. Notice from Step 2 that $\tilde{P}_{kl} > 0$. We now show that $\tilde{\Mat{P}}$ is, in fact, not optimal. To do so, we make a change in the four allocations, for some $\delta >0$ as small as desired,
    \begin{itemize}
\item we change $\tilde{P}_{ij}$ by $\left(\tilde{P}_{ij}^{2-t} - \delta \right)^{\frac{1}{2-t}}$;
\item we change $\tilde{P}_{il}$ by $\left(\tilde{P}_{il}^{2-t} + \delta \right)^{\frac{1}{2-t}}$;
\item we change $\tilde{P}_{kj}$ by $\left(\tilde{P}_{kj}^{2-t} + \delta \right)^{\frac{1}{2-t}}$;
\item we change $\tilde{P}_{kl}$ by $\left(\tilde{P}_{kl}^{2-t} - \delta \right)^{\frac{1}{2-t}}$.
\end{itemize}
Note that the marginal constraints are still satisfied. We show that the regularized measured cost decreases by a non-zero amount. To do so, we compute the difference $\Delta$ between the cost before and the cost after transformation:
  \begin{eqnarray*}
    \Delta(\delta) & = & \tilde{P}_{ij} M'_{ij} + \tilde{P}_{il} M'_{il} + \tilde{P}_{kj} M'_{kj} + \tilde{P}_{kl} M'_{kl} \nonumber\\
           & & - \left(\tilde{P}_{ij}^{2-t} - \delta \right)^{\frac{1}{2-t}}M'_{ij} - \left(\tilde{P}_{il}^{2-t} + \delta \right)^{\frac{1}{2-t}}M'_{il} - \left(\tilde{P}_{kj}^{2-t} + \delta \right)^{\frac{1}{2-t}}M'_{kj}-\left(\tilde{P}_{kl}^{2-t} - \delta \right)^{\frac{1}{2-t}}M'_{kl} \nonumber\\
           & = & \left(\tilde{P}_{ij}  - \left(\tilde{P}_{ij}^{2-t} - \delta \right)^{\frac{1}{2-t}}\right) \cdot M'_{ij}\\
           & & + \left(\left(\tilde{P}_{il}^{2-t} + \delta \right)^{\frac{1}{2-t}} - \tilde{P}_{il} \right)\cdot (-M'_{il})+ \left(\left(\tilde{P}_{kj}^{2-t} + \delta \right)^{\frac{1}{2-t}} - \tilde{P}_{kj} \right)\cdot (-M'_{kj})\\
           & & - \left( \tilde{P}_{kl}  - \left(\tilde{P}_{kl}^{2-t} - \delta \right)^{\frac{1}{2-t}}\right)\cdot (-M'_{kl}).
  \end{eqnarray*}
  We note that all terms but the last one are positive. We have from the generalized means inequality (since $1/(2-t) \leq 1$), for any $a \geq 0, 0\leq \delta \leq a^{2-t}$,
  \begin{eqnarray*}
    \left(\frac{(a^{2-t}+\delta)^{\frac{1}{2-t}}+(a^{2-t}-\delta)^{\frac{1}{2-t}}}{2}\right)^{2-t} & \leq & \frac{a^{2-t}+\delta + a^{2-t}-\delta}{2}\\
    & & = a^{2-t},
  \end{eqnarray*}
  which gives $a - (a^{2-t}-\delta)^{\frac{1}{2-t}} \geq (a^{2-t}+\delta)^{\frac{1}{2-t}} - a$. Let $g(z) \defeq (z^{2-t} + \delta)^{\frac{1}{2-t}} - z$. $\delta>0$ not being a function of $z$, $g$ is convex:
  \begin{eqnarray}
\frac{\partial^2 g}{\partial z^2} = (1-t) \delta z^{-t} \left(z^{2-t}+\delta\right)^{\frac{2t-3}{2-t}}.
  \end{eqnarray}
  This yields the following lower-bounds for $\Delta$, using $M' \defeq M'_{ij} - M'_{il}-M'_{kj}$ and $\Sigma \defeq \tilde{P}_{ij} M'_{ij} + \tilde{P}_{il}(-M'_{il}) + \tilde{P}_{kj}(-M'_{kj})$,
  \begin{eqnarray*}
    \Delta(\delta) & \geq & g(\tilde{P}_{ij}) \cdot M'_{ij} + g(\tilde{P}_{il}) \cdot (-M'_{il}) + g(\tilde{P}_{kj}) \cdot (-M'_{kj}) + \left(\left(\tilde{P}_{kl}^{2-t} - \delta \right)^{\frac{1}{2-t}} - \tilde{P}_{kl} \right)\cdot (-M'_{kl})\\
           & \geq & M' \cdot g\left(\frac{\Sigma}{M'}\right)+ \left(\left(\tilde{P}_{kl}^{2-t} - \delta \right)^{\frac{1}{2-t}} - \tilde{P}_{kl} \right)\cdot (-M'_{kl})\\
           & & = \left(\left(\tilde{P}_{ij} M'_{ij} + \tilde{P}_{il}(-M'_{il}) + \tilde{P}_{kj}(-M'_{kj})\right)^{2-t} + \delta (M'_{ij}\right.\\
           & & \left.- M'_{il}-M'_{kj})^{2-t}\right)^{\frac{1}{2-t}} - (\tilde{P}_{ij} M'_{ij} + \tilde{P}_{il}(-M'_{il}) + \tilde{P}_{kj}(-M'_{kj})) \\
           & & - \left( \tilde{P}_{kl}  - \left(\tilde{P}_{kl}^{2-t} - \delta \right)^{\frac{1}{2-t}}\right)\cdot (-M'_{kl})\\
    & = & M' \cdot \left(\left(\frac{\Sigma}{M'}\right)^{2-t} + \delta \right)^{\frac{1}{2-t}} + (-M'_{kl})\cdot \left(\tilde{P}_{kl}^{2-t} - \delta \right)^{\frac{1}{2-t}} - \left(\Sigma + \tilde{P}_{kl}(-M'_{kl})\right)
  \end{eqnarray*}
  Notice the general analytical shape of this lower-bound:
  \begin{eqnarray*}
\Delta(\delta) & \geq & \underbrace{a_1 (b_1^{2-t} + \delta)^{\frac{1}{2-t}} + a_2 (b_2^{2-t} - \delta)^{\frac{1}{2-t}} - (a_1b_1 + a_2b_2)}_{\defeq L(\delta)},
  \end{eqnarray*}
which still provides $L(0) = \Delta(0) = 0$. We have
    \begin{eqnarray}
L'(\delta) & = & \frac{a_1 (b_1^{2-t} + \delta)^{-\frac{1-t}{2-t}}-a_2 (b_2^{2-t} - \delta)^{-\frac{1-t}{2-t}}}{2-t}.
    \end{eqnarray}
We want to show a constraint on $a_1, a_2, b_1, b_2$ such that $L'(\delta) > 0$ for $\delta \in (0, u]$ with $u>0$, thus allowing to select $\delta > 0$ such that $\Delta(\delta) > 0$. We need $a_1 (b_2^{2-t} - \delta)^{\frac{1-t}{2-t}} \geq a_2 (b_1^{2-t} + \delta)^{\frac{1-t}{2-t}}$, which after simplification gives
    \begin{eqnarray*}
      \delta & \leq & \underbrace{\frac{a_1^{\frac{2-t}{1-t}}b_2^{2-t}-a_2^{\frac{2-t}{1-t}}b_1^{2-t}}{a_1^{\frac{2-t}{1-t}}+a_2^{\frac{2-t}{1-t}}}}_{=u}.
    \end{eqnarray*}
    To get $u$ strictly positive, we need $a_1b_2^{1-t} > a_2b_1^{1-t}$, that is, with the OT variables,
    \begin{eqnarray}
      \tilde{P}_{kl}^{1-t} & > & \frac{-M'_{kl}}{M'_{ij} - M'_{il}-M'_{kj}} \cdot \left(\frac{\tilde{P}_{ij} M'_{ij} + \tilde{P}_{il}(-M'_{il}) + \tilde{P}_{kj}(-M'_{kj})}{M'_{ij} - M'_{il}-M'_{kj}}\right)^{1-t}\nonumber\\
      & & = \frac{|M'_{kl}|}{|M'_{ij}| + |M'_{il}| + |M'_{kj}|} \cdot \left(\frac{\tilde{P}_{ij} |M'_{ij}| + \tilde{P}_{il}|M'_{il}| + \tilde{P}_{kj}|M'_{kj}|)}{|M'_{ij}| + |M'_{il}| + |M'_{kj}|}\right)^{1-t}\label{binfp1}.
      \end{eqnarray}
      So, if
      \begin{eqnarray*}
\tilde{P}_{kl}^{1-t} > \frac{|M'_{kl}|}{|M'_{ij}| + |M'_{il}| + |M'_{kj}|} \cdot \max\{\tilde{P}_{ij}, \tilde{P}_{il}, \tilde{P}_{kj}\}^{1-t},
      \end{eqnarray*}
      then \eqref{binfp1} holds, allowing us to pick $\delta > 0$ which yields $\Delta(\delta) > 0$ and contradicts the optimality of $\tilde{\Mat{P}}$.\\

  Summarizing Steps 1-4, if the cost configuration $M'_{ij} > 0 \wedge M'_{il} < 0 \wedge M'_{kj} < 0$ happens, then $\tilde{P}_{il} > 0$, $\tilde{P}_{kj} > 0$ and at least one of $\tilde{P}_{ij}$ and $\tilde{P}_{kl}$ is zero, which is the claim of the theorem.
  \end{proof}

\esol*
\begin{proof}
Similarly, using Lagrange multipliers $\ve{\nu}, \ve{\gamma} \in \mathbb{R}^n$ and $\Mat{\Pi} \in \mathbb{R}_+^{n\times n}$, we can write the problem as
\[
\min_{\tilde{\Mat{P}} \in \mathbb{R}^{n\times n}}\Big\{ \innerproduct{\tilde{\Mat{P}}^{1/t^*}}{\Mat{M}} + \frac{1}{\lambda} \cdot D_t(\tilde{\Mat{P}} \| \tilde{\ve{r}}\tilde{\ve{c}}^\top) +  \ve{\nu}^\top(\tilde{\Mat{P}}^{1/t^*}\ve{1}_n - \tilde{r}^{1/t^*}) + \ve{\gamma}^\top(\tilde{\Mat{P}}^{\top1/t^*}\ve{1}_n - \tilde{c}^{1/t^*}) - \innerproduct{\Mat{\Pi}}{\tilde{\Mat{P}}}\Big\}\, .
\]
Setting the derivative to zero in terms of $\tilde{P}_{ij}$ (ignoring the constants throughout when convenient), we have
\begin{align*}
\tilde{P}_{ij}^{1-t}\, \big(1 + (1-t)\, (\nu_i + \gamma_j + \lambda M_{ij})\big) & = (\tilde{r}_i\tilde{c}_j)^{1-t} + \Pi_{ij}\, .
\end{align*}
KKT conditions impose $\Pi_{ij} \geq 0$, which prevents the factor of $\tilde{P}_{ij}^{1-t}$ from being negative and yield equivalently
\begin{align*}
\tilde{P}_{ij}^{1-t}\, \left[1 + (1-t)\, (\nu_i + \gamma_j + \lambda M_{ij})\right]_+ & = (\tilde{r}_i\tilde{c}_j)^{1-t} + \Pi_{ij}\, .
\end{align*}
We now proceed as in the proof of Theorem \ref{th:msol}, remarking that the factor of $\tilde{P}_{ij}$ cannot be zero, then dividing:
\begin{align*}
\tilde{P}_{ij}^{1-t}\,  & = \frac{(\tilde{r}_i\tilde{c}_j)^{1-t}}{\left[1 + (1-t)\, (\nu_i + \gamma_j + \lambda M_{ij})\right]_+} + \frac{\Pi_{ij}}{\left[1 + (1-t)\, (\nu_i + \gamma_j + \lambda M_{ij})\right]_+}\, .
\end{align*}
Again, the KKT conditions impose if $\Pi_{ij} > 0$, then $\tilde{P}_{ij} = 0$, yielding
\begin{align*}
  \tilde{P}_{ij}\,  & = \frac{\tilde{r}_i\tilde{c}_j}{\left[1 + (1-t)\, (\nu_i + \gamma_j + \lambda M_{ij})\right]_+^{\frac{1}{1-t}}} \\
  & = \frac{\tilde{r}_i\tilde{c}_j}{\exp_t\left(\nu_i + \gamma_j + \lambda M_{ij}\right)},
\end{align*}
as claimed.
\end{proof}


\iters*
\begin{proof}
We start by proving~\eqref{eq:row-proj} and~\eqref{eq:col-proj} follows similarly. Consider the Lagrangian form of the row projection problem
\begin{equation}
    \label{eq:row-proj-lagrange}
    \min_{\tilde{\Mat{P}} \in \mathbb{R}^{n\times n}} D_t(\tilde{\Mat{P}} \| \tilde{\Mat{P}}_{\!\circ}) + \ve{\nu}^\top\big(\tilde{\Mat{P}}^{1/t^*}\ve{1}_n - \tilde{\ve{r}}^{1/t^*}\big) - \innerproduct{\Mat{\Pi}}{\tilde{\Mat{P}}}\,,
\end{equation}
where $\ve{\nu}\in\mathbb{R}^n$ and $\Mat{\Pi}\in\mathbb{R}^{n\times n}_+$ are Lagrange multipliers. Setting the derivative of~\eqref{eq:row-proj-lagrange} w.r.t. $\tilde{\Mat{P}}$ to zero yields (ignoring the constants when convenient)
\[
    \log_t \tilde{P}_{ij} - \log_t \tilde{P}_{\!\circ ij} + \tilde{P}_{ij}^{1-t}\nu_i - \Pi_{ij} = 0\,.
\]
Expanding the terms, rearranging, and taking the $\frac{1}{1-t}$-power of both sides, we have
\begin{equation}
\label{eq:row-scaling}
\tilde{P}_{ij} = \underbrace{\exp_t(\nu_i)}_{\defeq\tilde{\nu}_i\geq 0}\tilde{P}_{\!\circ ij}\,.
\end{equation}
Note that~\eqref{eq:row-scaling} is a non-negative scaling of each row of $\tilde{\Mat{P}}_{\!\circ}$. Thus, the constraint is satisfied when $\tilde{\nu}_i = \tilde{r}_i / \tilde{\mu}_i$ where $\tilde{\mu}_i = (\sum_j \tilde{P}_{ij}^{1/t^*})^{t^*}$.
\end{proof}
\begin{proof}
The zero values in the seeds are a result of the behavior of $\exp_t$: for $t < 1$, $\exp_t(x) = 0$ for $x < -\sfrac{1}{(1 - t)}$. Thus, some elements of the seed~\eqref{eq:seed} become zero for large enough values of $\lambda$. Similarly, $\exp_t(x) = \infty$ for $t > 1$ and $x > \sfrac{1}{(t - 1)}$. Thus, some elements of the seed~\eqref{eq:seed-e} also become zero for large enough values of $\lambda$. Since $\exp_t$ is monotonic for all $t \in \mathbb{R}$, the zero elements can only increase with $\lambda$. Thus, the zero elements induced by $\lambda_1$ is a subset of the zero elements induced by $\lambda_2$ for $\lambda_1 < \lambda_2$.
\end{proof}
\begin{proof}
The diagonal equivalence of the non-negative matrix $\tilde{\Mat{K}}$ to $\tilde{\Mat{P}} \in \tilde{U}_n(\tilde{\ve{r}}, \tilde{\ve{c}})$ means
\begin{equation}
    \tilde{\Mat{P}} = \diag(\ve{\nu})\,\tilde{\Mat{K}}\,\diag(\ve{\gamma})\,,
\end{equation}
for some $\ve{\nu}, \ve{\gamma} \in \mathbb{R}^n_{++}$. Note that since the zero-pattern of $\tilde{\Mat{K}}$ does not change by diagonal scaling, $\tilde{\Mat{P}}$ has the same zero-pattern as $\tilde{\Mat{K}}$. Applying the monotonic mapping $\tilde{\Mat{P}} \mapsto \tilde{\Mat{P}}^{1/t^*}$ with $t^* > 0$ on both sides, we have
\begin{equation}
    \Mat{P} = \diag(\ve{\nu})^{1/t^*}\,\tilde{\Mat{K}}^{1/t^*}\,\diag(\ve{\gamma})^{1/t^*}\,,
\end{equation}
which implies the diagonal equivalence of $\tilde{\Mat{K}}^{1/t^*}$ to $\Mat{P}$. The reverse statement also holds by the same argument for the mapping $\Mat{P} \mapsto \Mat{P}^{t^*}$ for $t^* > 0$.

\end{proof}

\begin{figure*}[t!]
\begin{center}
    \subfigure{\includegraphics[width=0.27\linewidth]{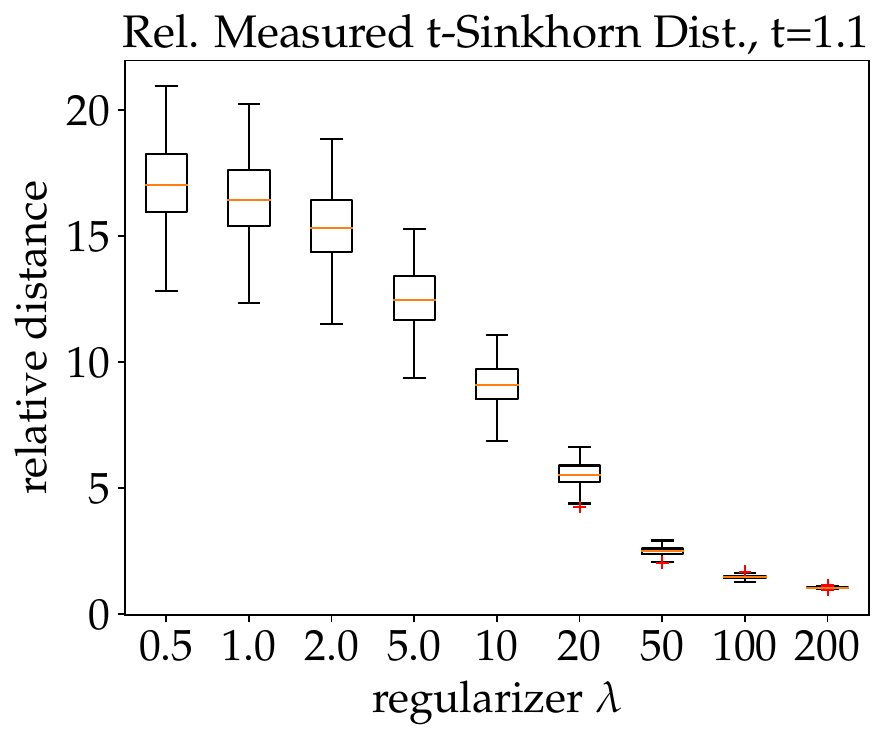}}
    \subfigure{\includegraphics[width=0.268\linewidth]{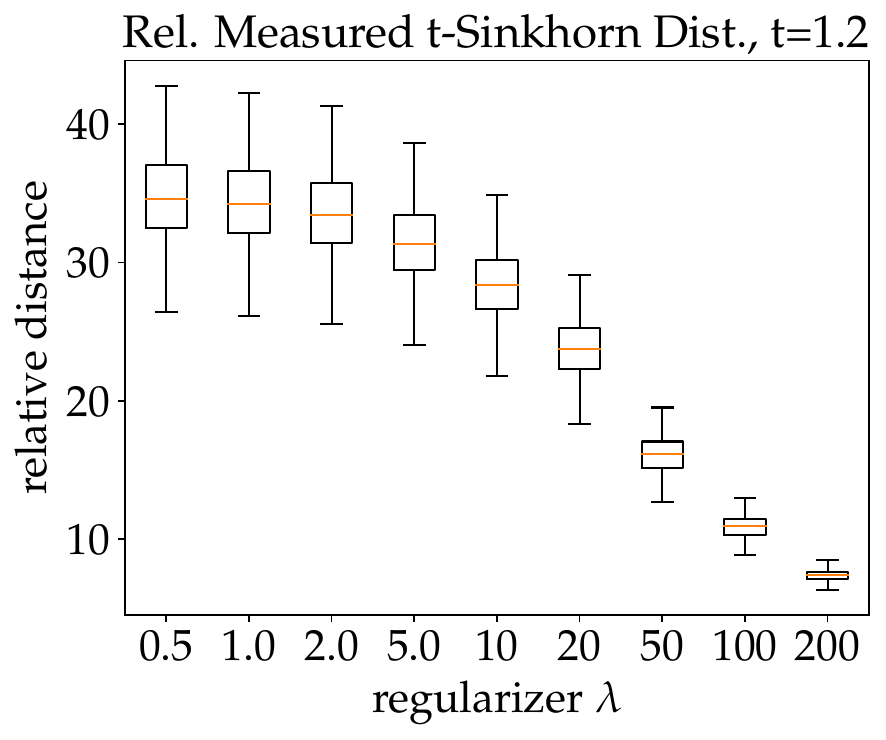}}
    \subfigure{\includegraphics[width=0.274\linewidth]{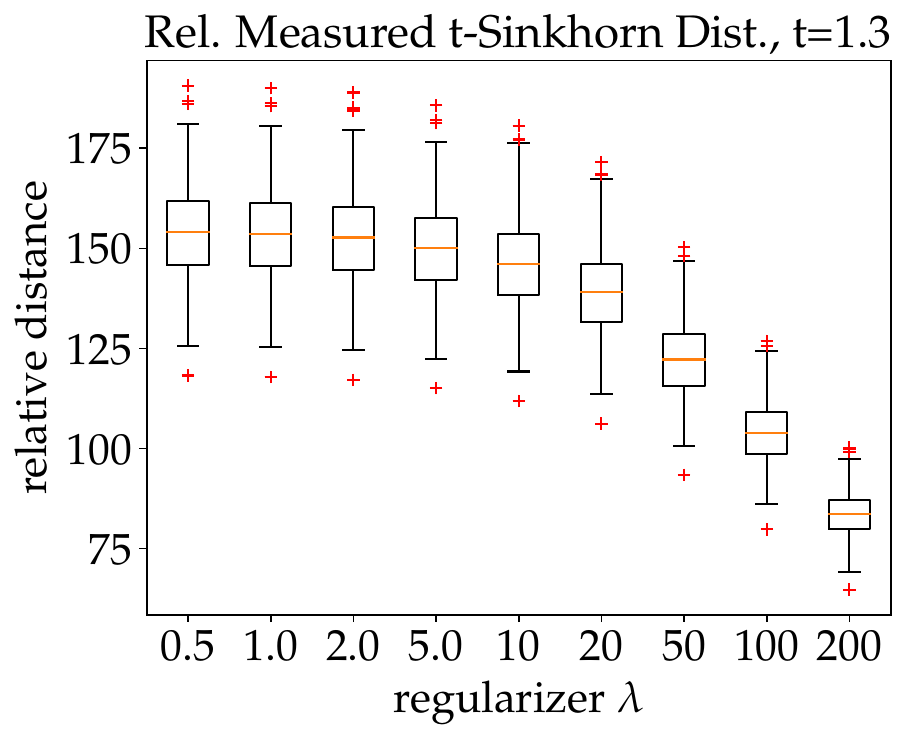}}
    \caption{The measured $t$-Sinkhorn distance relative to the Sinkhorn distance for different values of $t$. As $t\rightarrow 1$, the approximation error due to solving the unconstrained problem via alternating Bregman projections becomes smaller, and the measured $t$-Sinkhorn distance converges to the Sinkhorn distance for when $\lambda \rightarrow \infty$.}
    \label{fig:measued-dist}
    \end{center}
\end{figure*}

\section{Additional Experimental Results}

\subsection{Contraction Ratios}
We compute the contraction ratio (with respect to Hilbert's projective metric) of the seed matrix $\tilde{\Mat{K}}^{1/t^*}$ of Sinkhorn's iterations. We set $n=64$ and average the results over $100$ trials. For each case, we choose a range of $t$ such that the corresponding seed matrix is positive, namely, $t \in [1, 2)$ for the measured cost and $t < 1$ for the expected cost. We plot the contraction ratios of $\tilde{\Mat{K}}_m^{1/t^*}$ and $\tilde{\Mat{K}}_e^{1/t^*}$ in Figure~\ref{fig:contraction} (a) and (b), respectively, for different values of the regularizer $\lambda$. In both cases, the contraction ratio increases with $\lambda$, as we expect, resulting in slower convergence. The behavior w.r.t. $t$ is monotonic for the measured case; the convergence improves with larger $t$. The case of expected loss is more involved; the contraction ratio is higher for $t < 1$ compared to $t = 1$ (i.e., EOT) for small values of the regularizer. However, the contraction ratio increases at a slower rate for $t < 1$ when increasing $\lambda$, thus yielding improved convergence compared to EOT for larger values of $\lambda$. We verify the convergence results empirically in the next section.

\begin{figure}[t!]
\begin{center}
    \subfigure[]{\includegraphics[width=0.25\linewidth]{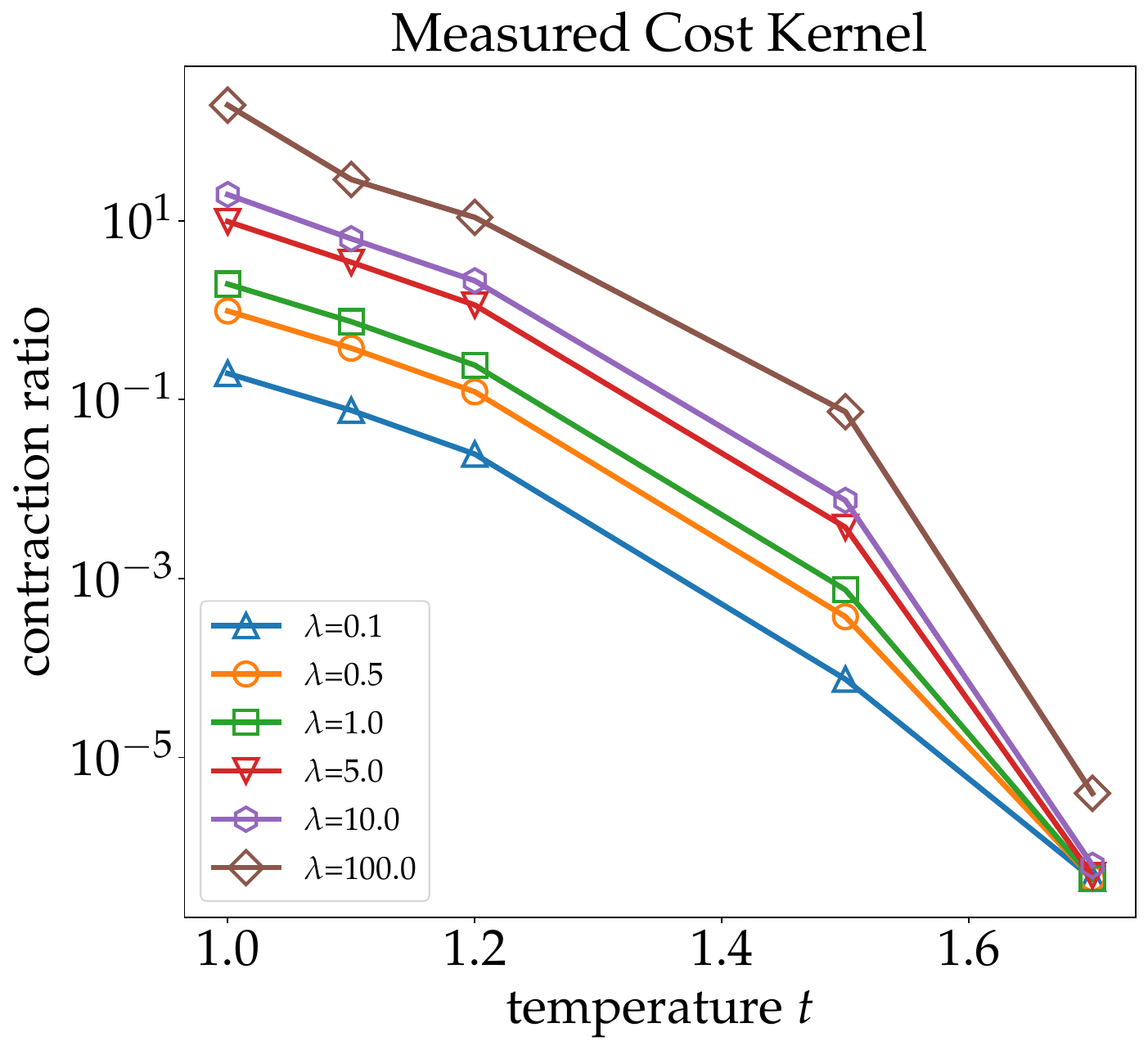}}
    \subfigure[]{\includegraphics[width=0.25\linewidth]{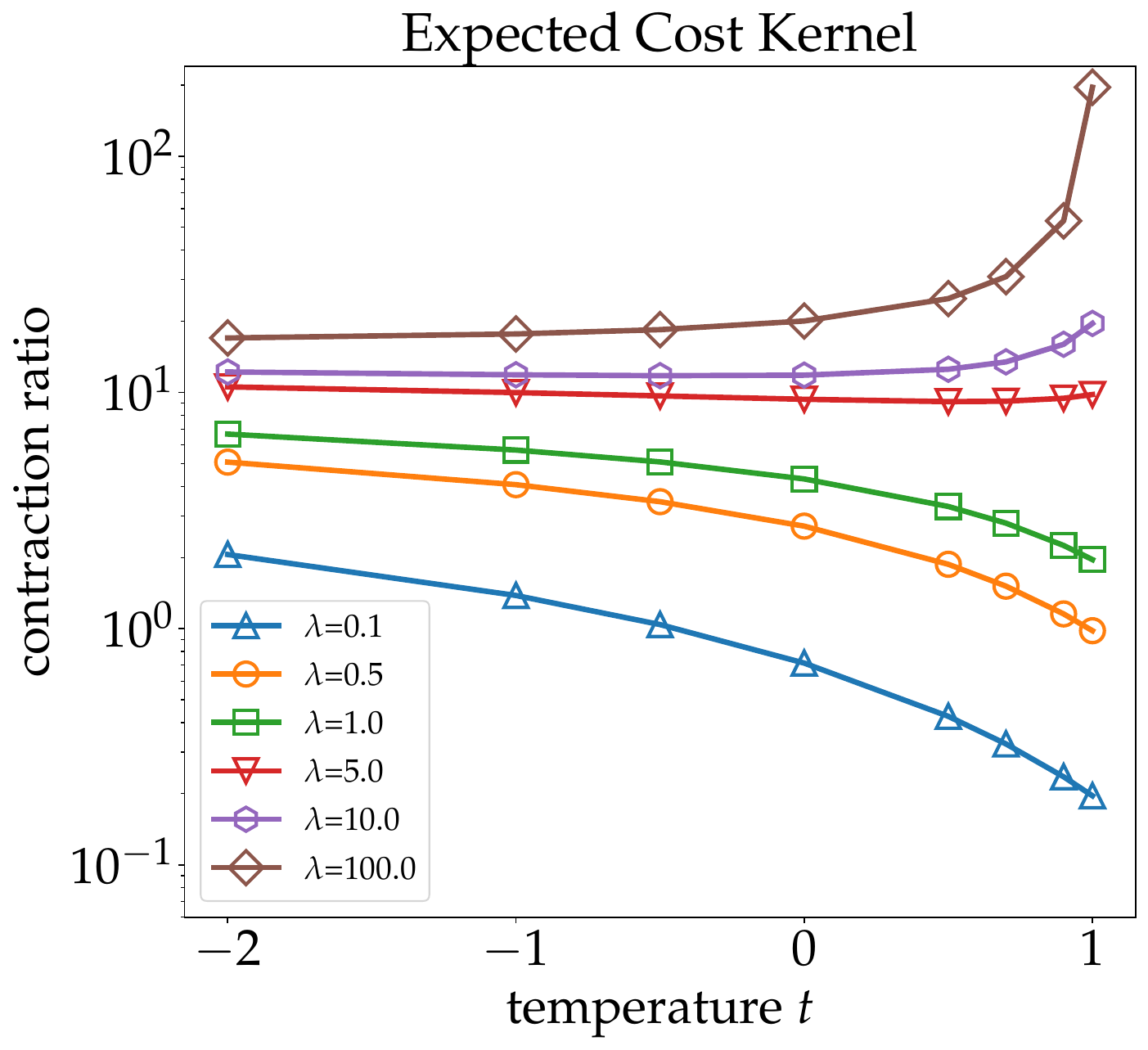}}
    \caption{Contraction ratios of (the $1/t^*$-th power of) (a) the measured $\tilde{\Mat{K}}_m^{1/t^*}$ and (b) the expected $\tilde{\Mat{K}}_e^{1/t^*}$ seed matrices. }
    \label{fig:contraction}
    \end{center}
\end{figure}

\subsection{Quality of Sinkhorn's Approximate Solution}
We analyze the error of the approximate solutions to~\eqref{eq:ot_sol_m_ex} and \eqref{eq:ot_sol_e_ex} via Sinkhorn's iterations. For this set of experiments, we set $n=64$ and average each result over $100$ trials. To calculate the true solutions, we use gradient descent to minimize the squared error and solve for $\{\nu_i\}$ and $\{\gamma_j\}$ such that $\tilde{\Mat{P}}$ is normalized along rows and columns. We calculate the expected cost of the true solution and the approximate solution via Sinkhorn and plot the relative error in Figure~\ref{fig:quality}. The quality of the approximate solution becomes relatively closer to the true solution for the case of expected cost formulation as $t$ becomes closer to one. However, the two solutions are closer in terms of quality for larger $t$ in the case of the measured cost. This can be attributed to the quality of our gradient descent solver for calculating the row and column (Lagrange multiplier) coefficients. For both cases, the approximation is better for smaller $\lambda$.

\begin{figure}[t!]
\begin{center}
    \subfigure[measured]{\includegraphics[width=0.25\linewidth]{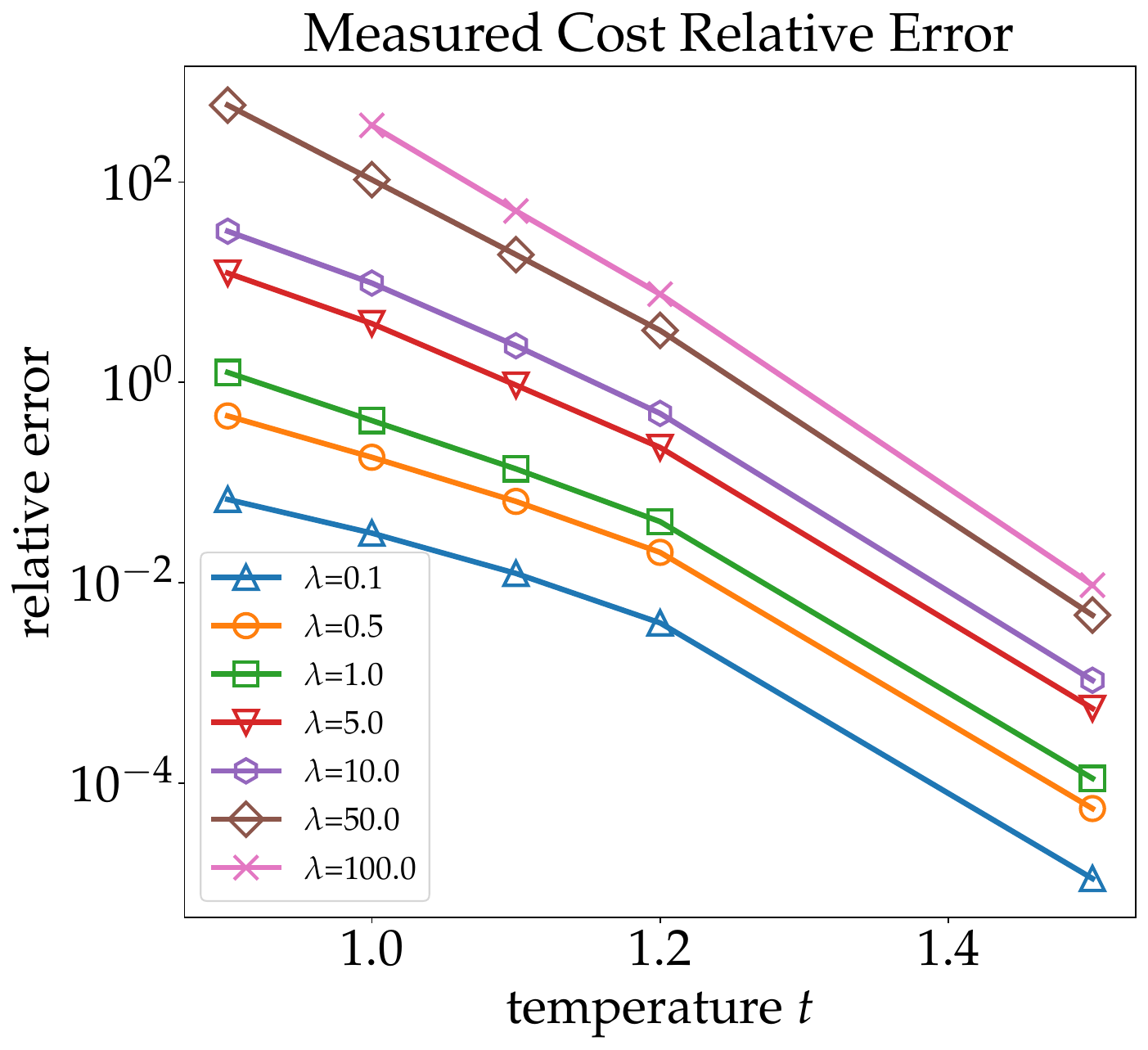}}
    \subfigure[expected]{\includegraphics[width=0.25\linewidth]{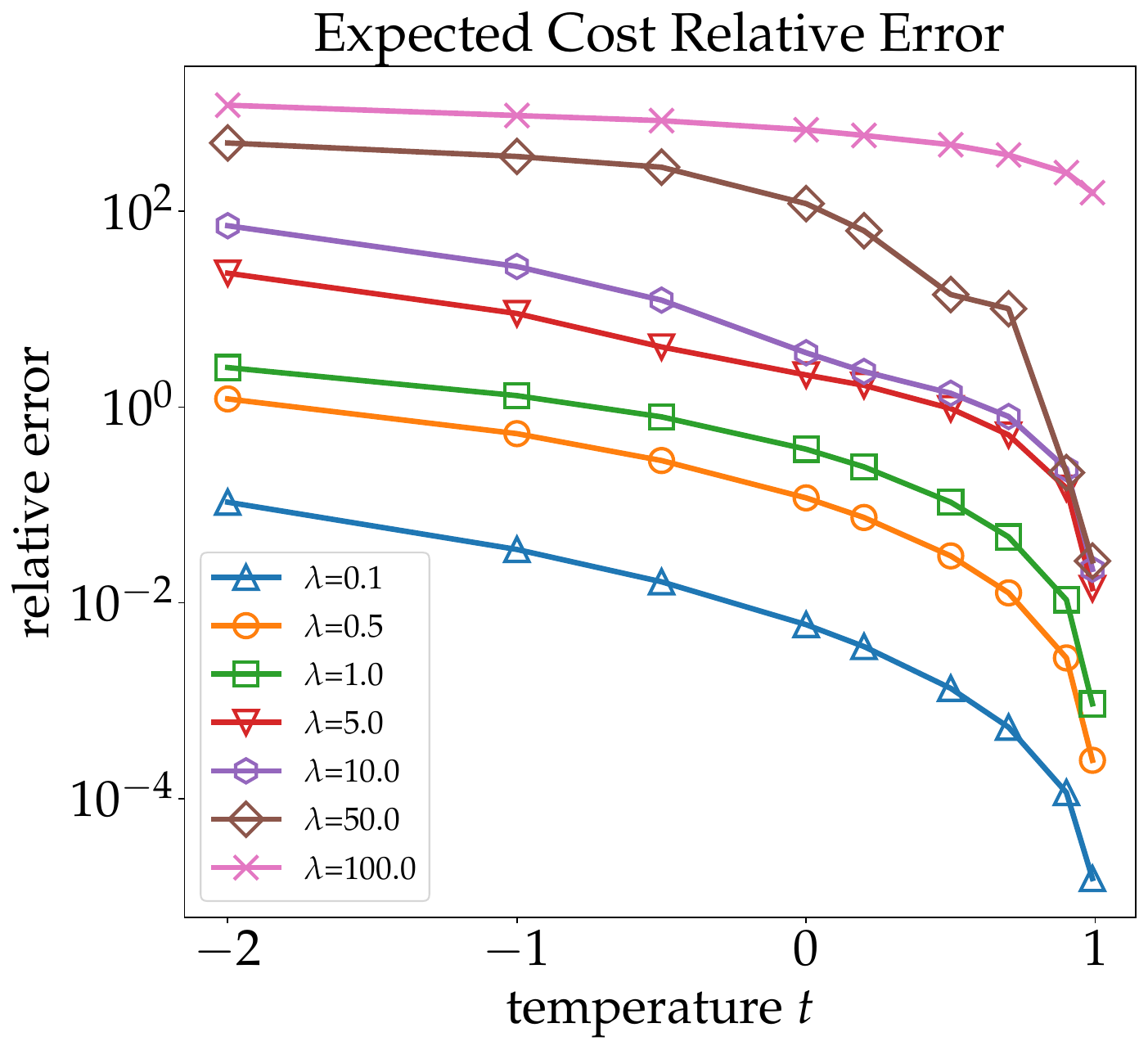}}
    \caption{Quality of the Sinkhorn's approximate solution: relative error of the Sinkhorn's solution to the solution obtained via gradient descent.}
    \label{fig:quality}
    \end{center}
\end{figure}

\begin{figure*}[t!]
\begin{center}
    \includegraphics[width=0.65\linewidth]{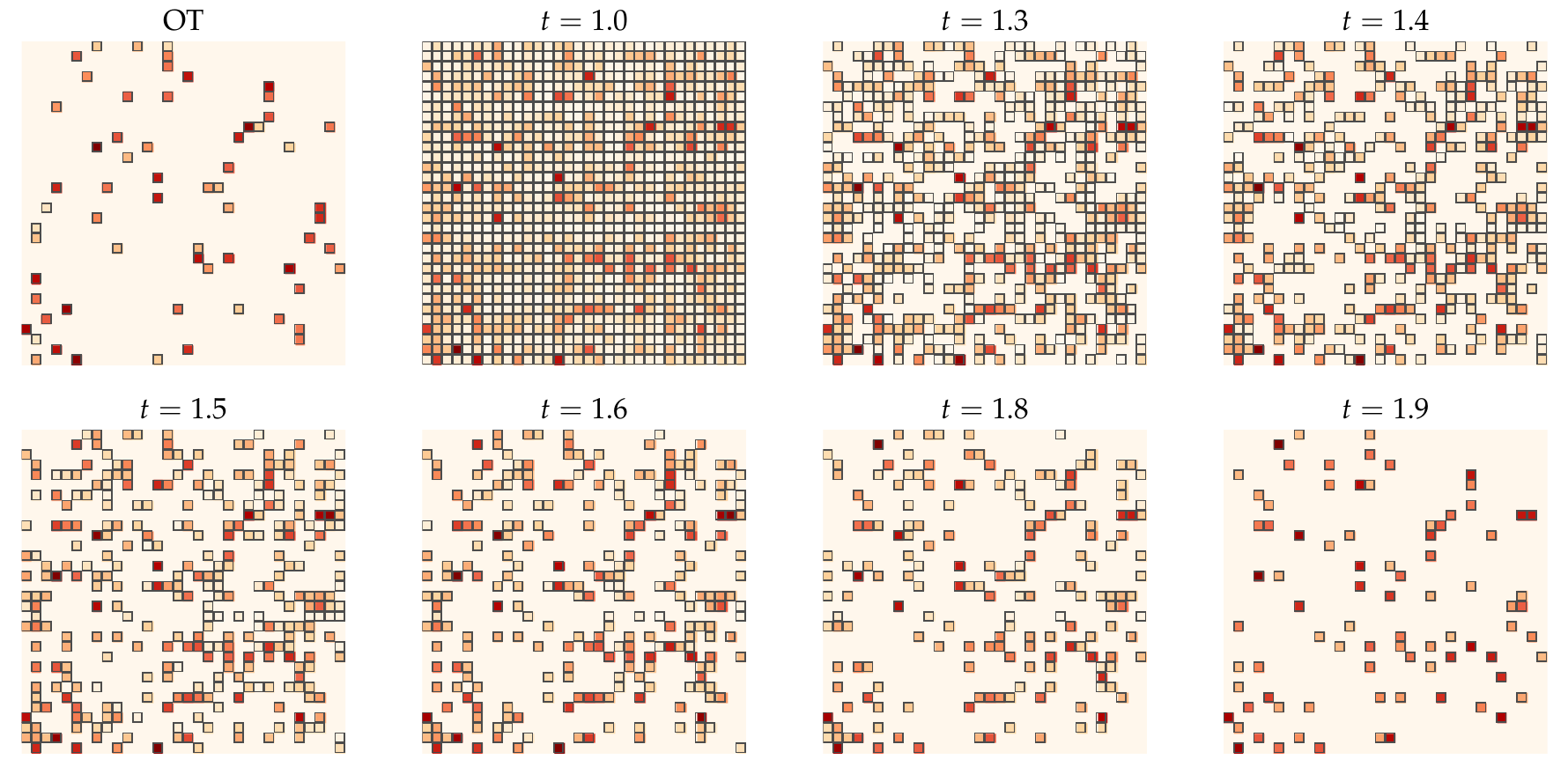}
    \caption{Transport plans induced by OT and the expected cost formulation for different values of $t$. The non-zero values are marked by a square. The EOT ($t=1$) induces a fully-dense plan. The sparsity of the solution increases by increasing $1 < t < 2$.}
    \label{fig:sparsity-s}
    \end{center}
\end{figure*}

\begin{figure*}[t!]
\begin{center}
    \includegraphics[width=0.65\linewidth]{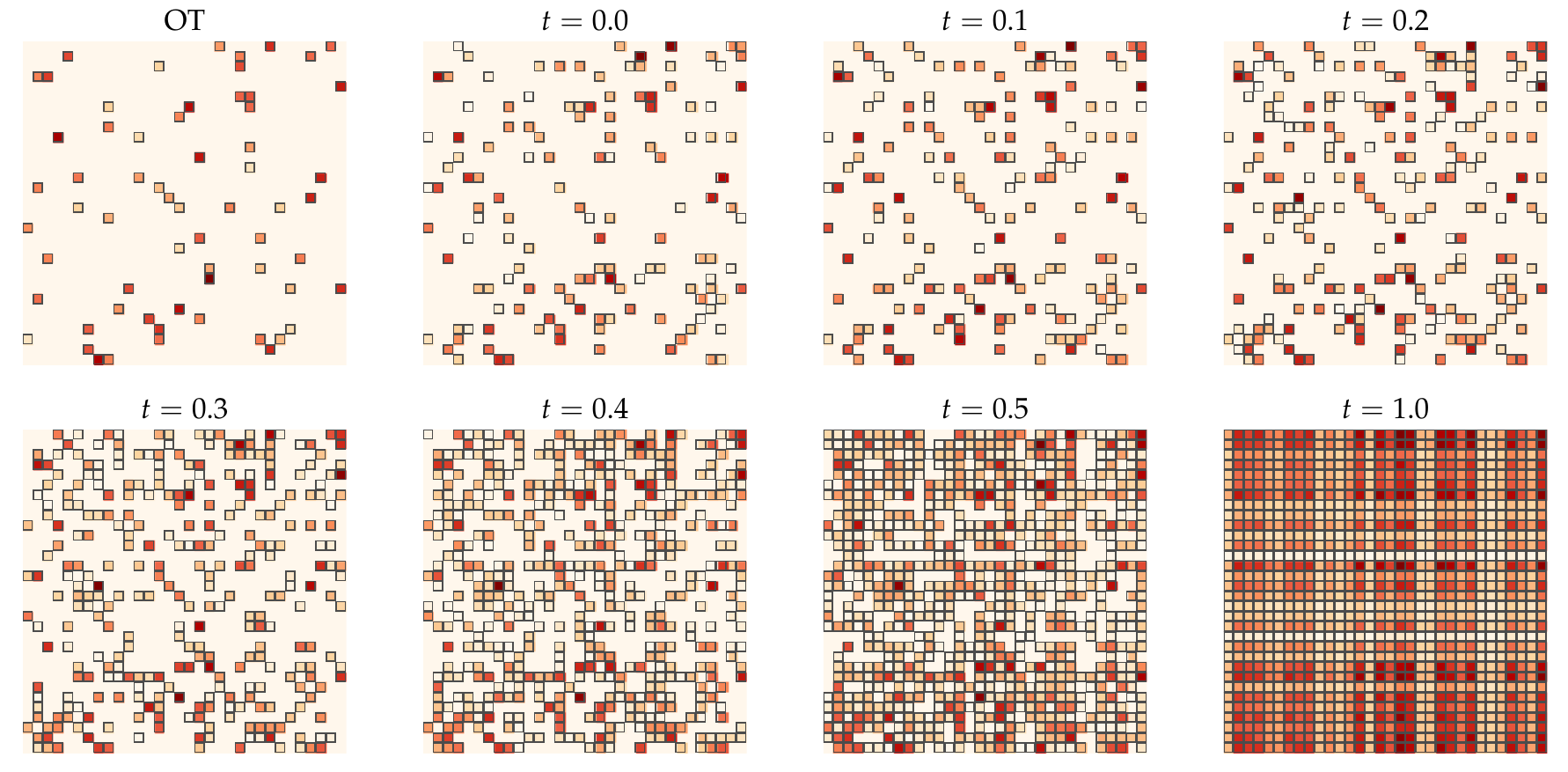}
    \caption{Transport plans induced by OT and the measured cost formulation for different values of $t$. The non-zero values are marked by a square. The EOT ($t=1$) induces a fully-dense plan. The sparsity of the solution increases by decreasing $t < 1$.}
    \label{fig:sparsity-m}
    \end{center}
\end{figure*}

\subsection{Sparse Solutions}
We provide expanded sparsity results along the lines of the ones presented in Section~\ref{subsec:sparsity} in Figure~\ref{fig:sparsity-s} and Figure~\ref{fig:sparsity-m}. For the measured cost sparsity results, we use $t\in [0, 1]$ and set $\lambda = 0.25$. In contrast to the expected cost, the sparsity of the solution for the case of the measured cost increases as we decrease the value of $t$.

\section{Remarks on Feasibility}

Let $\Mat{S}$ be the support indicator matrix of a non-negative matrix $\Mat{K}$. We say $\Mat{S}$ is \emph{indecomposable}\footnote{Commonly known as \emph{irreducible}~\citep{horn1990}.} if the rows and columns of $\Mat{S}$ cannot be permuted to form a block diagonal matrix. Without loss of generality, we assume that $\Mat{S}$ is indecomposable. \citet{brualdi1968convex} provides a necessary and sufficient condition for diagonal equivalence of a non-negative matrix $\Mat{K}$ to a matrix in $U_n(\ve{r}, \ve{c})$. 
\begin{theorem}{\citep{brualdi1968convex}}
\label{thm:brualdi}
If $\Mat{S}$ is an indecomposable zero pattern of a non-negative matrix $\Mat{K}$ and $\ve{r}$ and $\ve{c}$ are given positive probability vectors, then $\Mat{K}$ is diagonally equivalent to a matrix $\Mat{P} \in U_n(\ve{r}, \ve{c})$ if and only if the following holds: Whenever the rows and columns of $\Mat{S}$ can be permuted to the form 
\[
\Mat{S} = \begin{bmatrix}
\Mat{S}_1 & \Mat{0} \\
\Mat{S}_{21} & \Mat{S}_2
\end{bmatrix}\,,
\]
where $\Mat{S}_1$ is a non-vacuous 0, 1-matrix formed from rows $[I]$ and columns $[J]$ of $\Mat{S}$, and $\Mat{S}_2$ is non-vacuous, then
\[
c_1 + \ldots + c_J > r_1 + \ldots + r_I\, .
\]
\end{theorem}
Unfortunately, the result of~\citet{brualdi1968convex} does not provide an easy ``test'' for the feasibility of the sparse-case solution: checking indecomposablity is equivalent to solving the strongly connected component problem~\citep{horn1990} on the zero-patterns $\Mat{S}$, for which the best-known algorithm, namely, Tarjan's algorithm~\citep{tarjan1972depth}, has complexity $\mathcal{O}(\sum_{ij} S_{ij})$ (since each row and column must have at least one non-zero element; otherwise, the solution is infeasible). Finding permutations of rows and columns to obtain a block upper-triangular matrix is a significantly more challenging problem; \citet{fertin2015obtaining} show that even obtaining a triangular matrix by independent row-column permutations is NP-complete and provide an exponential-time algorithm for solving the problem. 
\end{document}